\newif\ifIEEEFORMAT
\IEEEFORMATfalse

\ifIEEEFORMAT
  \documentclass[letterpaper,10pt,conference]{ieeeconf}
  \IEEEoverridecommandlockouts
\else
  \documentclass[10pt,journal,onecolumn]{IEEEtran}
\fi

\usepackage[T1]{fontenc}
\usepackage[utf8]{inputenc}

\ifIEEEFORMAT
  \usepackage{amsmath,amssymb,amsfonts}

  \usepackage{amsthm}
\else
  \usepackage{amsmath,amssymb,amsfonts,amsthm}
\fi

\usepackage{newtxtext,newtxmath}
\usepackage{microtype}
\usepackage{bm}
\usepackage[hidelinks]{hyperref}
\ifIEEEFORMAT
\else
  \usepackage{titlesec}
\fi
\usepackage{cite}
\usepackage{graphicx}
\usepackage{array}
\usepackage{enumitem}
\setlist[enumerate]{leftmargin=*}
\setlist[itemize]{leftmargin=*}
\usepackage{booktabs}
\usepackage{tabularx}
\usepackage{etoolbox}
\usepackage{xparse}

\usepackage{tikz}
\usetikzlibrary{
  arrows.meta,
  calc,
  positioning,
  fit,
  backgrounds,
  decorations.pathreplacing
}

\definecolor{Navy}{HTML}{15324A}
\definecolor{Blue}{HTML}{1F6F9C}
\definecolor{Teal}{HTML}{168477}
\definecolor{Orange}{HTML}{C66A1B}
\definecolor{Red}{HTML}{A63D40}
\definecolor{Ink}{HTML}{182026}
\definecolor{Mid}{HTML}{65727C}
\definecolor{Pale}{HTML}{F3F5F6}
\definecolor{PaleBlue}{HTML}{EAF2F7}
\definecolor{PaleOrange}{HTML}{FAEFE4}
\definecolor{LinkBlue}{HTML}{165F8A}

\colorlet{armgray}{Ink}
\colorlet{taskblue}{Blue}
\colorlet{ellipseorange}{Orange}

\tikzset{
  arm/.style={
    draw=armgray,
    line width=2.2pt,
    line cap=round,
    line join=round
  },
  joint/.style={
    circle,
    draw=armgray,
    fill=white,
    line width=0.8pt,
    inner sep=1.55pt
  },
  endpoint/.style={
    circle,
    draw=taskblue,
    fill=white,
    line width=1.0pt,
    inner sep=1.8pt
  },
  ellipsoid/.style={
    draw=ellipseorange,
    fill=ellipseorange,
    fill opacity=0.14,
    line width=0.9pt
  },
  taskaxis/.style={
    taskblue,
    -{Latex[length=2.0mm]},
    line width=0.8pt
  },
  paneltext/.style={font=\small},
  valuetext/.style={font=\scriptsize}
}

\newcommand{\V}{\mathcal{V}}
\newcommand{\W}{\mathcal{W}}
\newcommand{\Fw}{\mathcal{F}_{w}}
\newcommand{\R}{\mathbb{R}}
\newcommand{\dd}{\mathrm{d}}
\newcommand{\argmax}{\operatorname*{arg\,max}}
\newcommand{\argmin}{\operatorname*{arg\,min}}

\newtheorem{theorem}{Theorem}
\newtheorem{lemma}{Lemma}
\newtheorem{corollary}{Corollary}
\newtheorem{definition}{Definition}
\newtheorem{proposition}{Proposition}

\newtheorem{asmpt}{Assumption}

\ifIEEEFORMAT
\else
  \renewcommand{\thesection}{\arabic{section}}
  \renewcommand{\thesubsection}{\thesection.\arabic{subsection}}

  \titleformat{\section}
    {\normalfont\large\bfseries}{\thesection}{0.65em}{}
  \titleformat{\subsection}
    {\normalfont\normalsize\bfseries}{\thesubsection}{0.65em}{}

  \titlespacing*{\section}
    {0pt}{1.8ex plus 0.4ex minus 0.2ex}{0.6ex}
  \titlespacing*{\subsection}
    {0pt}{1.25ex plus 0.3ex minus 0.2ex}{0.4ex}

\fi

\newcommand{\PaperAuthorsAffiliated}
  {Antonio Franchi$^{1,2}$ and Mirko Mizzoni$^{1}$}

\newcommand{\AffiliationOne}{%
  $^1$ Robotics and Mechatronics Group, Faculty of Electrical Engineering,
  Mathematics and Computer Science, University of Twente, Enschede,
  The Netherlands, \texttt{schol@r-franchi.eu},
  \texttt{m.mizzoni@utwente.nl}.%
}

\newcommand{\AffiliationTwo}{%
  $^2$ Department of Computer, Control and Management Engineering,
  Sapienza University of Rome, 00185 Rome, Italy,
  \texttt{schol@r-franchi.eu}.%
}

\newcommand{\FundingStatement}{%
  This work was partially funded by the Horizon Europe research
  project AUTOASSESS under grant agreement No.~101120732.%
}

\newcommand{\PaperTitle}{}
\newcommand{\PaperKeywords}{}
\NewDocumentCommand{\shorttitle}{m}{}
\NewDocumentCommand{\shortauthors}{m}{}
\RenewDocumentCommand{\title}{o m}{\gdef\PaperTitle{#2}}
\RenewDocumentCommand{\author}{o m}{}
\NewDocumentCommand{\cormark}{o}{}
\NewDocumentCommand{\fnmark}{o}{}
\NewDocumentCommand{\ead}{o m}{}
\NewDocumentCommand{\affiliation}{o m}{}
\NewDocumentCommand{\cortext}{o m}{}
\NewDocumentCommand{\fntext}{o m}{}
\newcommand{\sep}{\unskip,\ }
\NewDocumentEnvironment{keywords}{+b}{\gdef\PaperKeywords{#1}}{}

\makeatletter
\RenewDocumentEnvironment{figure}{o}{%
  \IfNoValueTF{#1}
    {\@float{figure}}
    {\ifstrequal{#1}{pos=t}{\@float{figure}[t]}{\@float{figure}[#1]}}%
}{\end@float}
\RenewDocumentEnvironment{table}{o}{%
  \IfNoValueTF{#1}
    {\@float{table}}
    {\ifstrequal{#1}{pos=t}{\@float{table}[t]}{\@float{table}[#1]}}%
}{\end@float}
\makeatother

\ifIEEEFORMAT
\else
  \let\ArxivOriginalAbstract\abstract
  \let\endArxivOriginalAbstract\endabstract
  \RenewDocumentEnvironment{abstract}{}{%
    \ArxivOriginalAbstract
  }{%
    \endArxivOriginalAbstract
  }
\fi

\ifIEEEFORMAT
\else
  \makeatletter
  \def\maketitle{%
    \par\begingroup
      \raggedright
      {\LARGE\bfseries \PaperTitle\par}
      \vspace{0.7em}
      {\normalsize \PaperAuthorsAffiliated\par}
      \vspace{0.55em}
      {\footnotesize
        \AffiliationOne\par
        \AffiliationTwo\par
        \FundingStatement\par
      }
      \vspace{1.0em}
    \endgroup
  }
  \makeatother
\fi

\begin{document}

\let\WriteBookmarks\relax
\def\floatpagepagefraction{1}
\def\textpagefraction{.001}

\shorttitle{Fiber-Normalized Manipulability and Determinant Proxies}
\shortauthors{A. Franchi and M. Mizzoni}

\title[mode=title]{
Fiber-Normalized Manipulability and Determinant Proxies:
Intrinsic Redundancy Optimization Across and Within Task Fibers
}

\author[aff1,aff2]{Antonio Franchi}
\cormark[1]
\fnmark[1]
\ead{schol@r-franchi.eu}

\author[aff1]{Mirko Mizzoni}
\ead{m.mizzoni@utwente.nl}

\affiliation[aff1]{
  organization={Robotics and Mechatronics Group},
  department={
    Faculty of Electrical Engineering,
    Mathematics and Computer Science
  },
  university={University of Twente},
  city={Enschede},
  country={The Netherlands}
}

\affiliation[aff2]{
  organization={
    Department of Computer, Control and Management Engineering
  },
  university={Sapienza University of Rome},
  city={Rome},
  postcode={00185},
  country={Italy}
}

\cortext[cor1]{Corresponding author.}

\fntext[fn1]{
  This work was partially funded by the Horizon Europe research
  project AUTOASSESS under grant agreement No.~101120732.
}


\maketitle

\begin{abstract}
This work establishes that determinant-based manipulability is an exact
objective for fixed-task redundancy optimization, despite its dependence
on task coordinates and the choice of task-space metric used for volume
measurement. On every regular task fiber, the determinant proxy, its
representation in any task chart, and every metric-completed
manipulability differ only by positive constants. They consequently
induce the same complete ordering, constrained extrema, gradient
directions, critical points, and local optimality classifications.

For comparisons and trajectory optimization across task fibers, this
work introduces \emph{fiber-normalized manipulability}: the capability
attained at an internal state divided by the best capability available
on the same fiber. The resulting dimensionless scalar is invariant
under coordinate changes on the internal-state and task manifolds and
independent of the task-space metric. Its associated loss provides an
intrinsic objective for physically admissible cross-fiber trajectories,
including problems with prescribed or free task evolution and temporal
coupling. Planar-manipulator and redundant aerodynamic-allocation
examples demonstrate fixed-fiber equivalence, task-dependent
cross-fiber differences, and invariant fiber-normalized trajectory
optimization.
\end{abstract}

\section{Introduction}
\label{sec:introduction}

Kinematic redundancy arises when the configuration variables of a
robotic or mechanical system provide more degrees of freedom than are
required to describe its task. The same task value can then be realized
by multiple, and often infinitely many, configurations. This freedom
can be exploited to satisfy secondary objectives, including joint-limit
avoidance, obstacle avoidance, reduction of actuator effort, and
improvement of dexterity.

A seminal formulation of differential inverse kinematics for robot
control is Whitney's resolved-motion-rate method~\cite{Whitney1969},
which maps a desired task-space velocity into joint velocities through
the instantaneous inverse of the task Jacobian. For redundant systems,
the general solution of the differential kinematics contains a
homogeneous component that does not affect the primary task velocity. A
seminal step in its systematic exploitation was Li\'egeois'
gradient-projection method~\cite{Liegeois1977}, which uses this
component to project the gradient of a configuration-dependent
secondary criterion into the null space of the task Jacobian. Task-
priority schemes and related formulations subsequently extended this
principle to multiple objectives and constraints
~\cite{NakamuraHanafusaYoshikawa1987,Siciliano1990}. These developments
established the framework underlying much of contemporary redundancy
resolution: the primary task specifies the required output motion,
while the remaining mobility is assigned to secondary performance
criteria.

Manipulability became one of the most influential criteria for this
purpose. Yoshikawa introduced a Jacobian-based measure associated with
the volume of the task-space velocity ellipsoid and used it to assess
the ability of a mechanism to generate task motion
~\cite{Yoshikawa1985,Yoshikawa1985Control}. The same idea was extended
to dynamic manipulability~\cite{Yoshikawa1985Dynamic} and to separate
translational and rotational capabilities~\cite{Yoshikawa1990}. Other
dexterity and performance measures were subsequently investigated for
mechanism design, optimal posture selection, and control
~\cite{KleinBlaho1987,Doty1995}; a broad account of their definitions,
classifications, scope, and limitations is provided by the survey
in~\cite{PatelSobh2015}.

In the familiar Euclidean representation, volume-based manipulability
is computed from the determinant of the Jacobian Gram matrix. We call
this unweighted coordinate expression the \emph{determinant proxy}.
The question addressed here is whether its coordinate dependence
prevents it from being an exact objective for fixed-task redundancy
optimization.

Manipulability optimization has since been incorporated into many
forms of redundancy resolution. Representative developments include
dexterity-optimal motion control~\cite{Iwatsuki1994}, null-space and
optimization-based controllers~\cite{Jin2017,Su2019}, and globally
oriented redundancy-resolution methods for serial
manipulators~\cite{HuberWollherr2019,HuberWollherr2021}. Manipulability
has also been included in collision-free path planning
~\cite{KadenThomas2019,Shen2023} and continuous-time trajectory
optimization~\cite{Maric2019}. These studies demonstrate the practical
value of maintaining configurations with favorable differential task
capability during motion. More recently, related determinant-based
quantities have also been employed beyond conventional serial
manipulation, including redundant aerodynamic actuation and multirotor
control allocation
~\cite{Franchi2026Coactivation,Franchi2026AeroPromptness}.

The geometric interpretation of manipulability has also received
considerable attention. Coordinate-free formulations regard the
forward kinematics as a smooth map between Riemannian manifolds and
define dexterity through geometric structures on the configuration and
task spaces~\cite{ParkBrockett1994,ParkKim1998}. These formulations
clarify that manipulability is not determined by a Jacobian matrix
alone: the matrix represents a differential only after coordinates have
been selected, while lengths, volumes, and comparisons of heterogeneous
task components require suitable metric structures. More recent
tensorial analyses have made the corresponding transformation laws
explicit and have shown how metrics and tensor contractions yield
coordinate-invariant manipulability descriptions~\cite{Lachner2020}.

These geometric formulations distinguish a change of coordinates from
a change of metric. The determinant proxy depends on the task chart;
a metric-completed measure is chart invariant, but its numerical
value depends on the task metric chosen to measure volume. We examine
the consequences for two uses: optimizing redundant configurations
at a prescribed task value, and comparing configurations that realize
different task values.

Our first result establishes exact optimization equivalence on a
fixed task fiber. For comparisons across fibers, we introduce
\emph{fiber-normalized manipulability}, which evaluates each
configuration relative to the best capability available at the same
task value. This separates relative redundancy utilization from
absolute task capability.

The main contributions of this work are:
\begin{enumerate}[label=(\roman*)]
    \item A geometric characterization of fixed-fiber equivalence:
    task-side coordinate and metric changes preserve constrained
    extrema, restricted gradient directions, critical points, and
    local optimality classifications.
    \item A fiber-normalized index, with well-posedness and range
    results, that is invariant under coordinate changes on both
    manifolds and independent of the task-space metric.
    \item An analysis of cross-fiber and trajectory objectives,
    illustrated on redundant-manipulator and aerodynamic-allocation
    examples with and without temporal coupling.
\end{enumerate}

Section~\ref{sec:setting} introduces the geometric construction.
Section~\ref{sec:main-results} distinguishes the optimization regimes
and establishes fixed-fiber equivalence.
Section~\ref{sec:fiber-normalized} develops the normalized index and
its use in trajectory objectives. Section~\ref{sec:examples} presents
the numerical case studies, and Section~\ref{sec:discussion} discusses
their practical implications and limitations.

 \section{Geometric formulation of manipulability}
\label{sec:setting}

Determinant-based manipulability indices arise from three ingredients:
the differential kinematics of the task, a notion of admissible
internal velocity, and a notion of volume in the task space. This
section introduces these objects and clarifies the geometric meaning
of the matrix whose determinant is used in classical manipulability
indices.

\subsection{Task map, redundancy, and fixed-task fibers}
\label{subsec:task-fibers}

Let $(\V,g)$ be an $n$-dimensional Riemannian manifold representing
the configuration, actuator, or internal state space of a mechanical
system. Its Riemannian metric is a smooth assignment
\begin{align}
    v\longmapsto g_v,
\end{align}
equivalently represented by the smooth map
\begin{align}
    g:T\V\times_{\V}T\V\to\R,
\end{align}
such that, for every $v\in\V$,
\begin{align}
    g_v:T_v\V\times T_v\V\to\R
\end{align}
is an inner product, that is, a positive-definite symmetric bilinear
form.
The metric determines the magnitude of an internal tangent vector
$\xi\in T_v\V$ through
\begin{align}
    \|\xi\|_{g_v}^{2}:=g_v(\xi,\xi).
\end{align}
For a robotic manipulator, $g$ may be the Euclidean joint-rate metric.
More generally, it may account for nonuniform, coupled, or
state-dependent actuation capabilities.

Let $\W$ be an $m$-dimensional task manifold, with $n>m$, and let
\begin{align}
    f:\V\to\W
    \label{eq:map_f}
\end{align}
be a smooth task map. At $v\in\V$, with $w=f(v)$, its differential
\begin{align}
    \mathrm df_v:T_v\V\to T_w\W
\end{align}
maps tangent vectors at $v$ into tangent vectors at $w$. Depending on
the physical meaning of the manifold variables, these tangent vectors
may represent joint rates, actuator-state rates, propeller accelerations, or
other first-order variations.

We restrict the analysis to an open regular region on which
\begin{align}
    \operatorname{rank}(\mathrm df_v)=m.
    \label{eq:regular-task-map}
\end{align}
Thus, $\mathrm df_v$ is surjective throughout the region of interest.
Global surjectivity of $f$ is not required, and the regular region is
denoted again by $\V$ for notational simplicity.

For a prescribed task value $w\in f(\V)$, the admissible internal
states form the fiber
\begin{align}
    \Fw:=f^{-1}(w).
    \label{eq:fiber-definition}
\end{align}
By the regular-level-set theorem, $\Fw$ is an embedded submanifold of
dimension $n-m$, and
\begin{align}
    T_v\Fw=\ker(\mathrm df_v),
    \quad v\in\Fw.
    \label{eq:fiber-tangent-space}
\end{align}
Hence, a fixed-task fiber is the nonlinear feasible set underlying
redundancy resolution, while its tangent space is the Jacobian null
space used in its differential formulation.

At each $v\in\V$, the inner product $g_v$ induces the orthogonal
decomposition
\begin{align}
    T_v\V=V_v\oplus H_v,
    \label{eq:vertical-horizontal-decomposition}
\end{align}
where
\begin{align}
    V_v:=T_v\mathcal F_w=\ker(\mathrm df_v),
    \quad
    H_v:=V_v^{\perp_{g_v}}.
\end{align}
The vertical space $V_v$ contains the tangent directions that leave
the task value unchanged to first order. The horizontal space $H_v$
is their $g_v$-orthogonal complement and contains the tangent
directions responsible for first-order task variations.

\subsection{From internal tangent vectors to a task-space co-metric}
\label{subsec:pushed-cometric}

Manipulability describes how admissible tangent vectors at an internal
state are mapped into tangent vectors at the corresponding task value.
To formulate this relation geometrically, we first introduce the
co-metric induced by $g_v$ on the internal cotangent space.

For every $v\in\V$, the inner product $g_v$ defines the flat musical
isomorphism
\begin{align}
    g_v^\flat:T_v\V\to T_v^*\V,
    \quad
    g_v^\flat(\xi):=g_v(\xi,\cdot).
    \label{eq:internal-flat-map}
\end{align}
Since $g_v$ is positive definite, $g_v^\flat$ is an isomorphism. Its
inverse is the sharp musical isomorphism
\begin{align}
    g_v^\sharp
    :=
    \left(g_v^\flat\right)^{-1}
    :
    T_v^*\V\to T_v\V.
    \label{eq:internal-sharp-map}
\end{align}
Thus, every covector $\eta\in T_v^*\V$ is associated with the unique
tangent vector $g_v^\sharp(\eta)\in T_v\V$ satisfying
\begin{align}
    \eta(\xi)
    =
    g_v\!\left(g_v^\sharp(\eta),\xi\right)
    \quad
    \text{for every }\xi\in T_v\V.
    \label{eq:sharp-characterization}
\end{align}

The inner product $g_v$ induces the co-metric
\begin{align}
    g_v^*:T_v^*\V\times T_v^*\V\to\R
\end{align}
defined by
\begin{align}
    g_v^*(\eta,\zeta)
    :=
    g_v\!\left(
        g_v^\sharp(\eta),
        g_v^\sharp(\zeta)
    \right)
    =
    \zeta\!\left(g_v^\sharp(\eta)\right).
    \label{eq:internal-cometric}
\end{align}
Therefore, $g_v^*$ is an inner product on the cotangent space
$T_v^*\V$.

Now let $w=f(v)$. The differential of the task map induces the
pullback
\begin{align}
    \mathrm df_v^*:T_w^*\W\to T_v^*\V,
    \label{eq:task-covector-pullback}
\end{align}
defined, for every $\alpha\in T_w^*\W$, by
\begin{align}
    \mathrm df_v^*\alpha
    :=
    \alpha\circ\mathrm df_v.
    \label{eq:task-covector-pullback-definition}
\end{align}
Equivalently, for every $\xi\in T_v\V$,
\begin{align}
    \left(\mathrm df_v^*\alpha\right)(\xi)
    =
    \alpha\!\left(\mathrm df_v(\xi)\right).
    \label{eq:pullback-evaluation}
\end{align}
Thus, $\mathrm df_v^*\alpha$ is the internal covector that evaluates
an internal tangent vector according to the task variation produced
by that tangent vector.

The internal co-metric $g_v^*$ can consequently be used to compare the
pullbacks of task covectors. For $\alpha,\beta\in T_w^*\W$, define
\begin{align}
    d_v(\alpha,\beta)
    :=
    g_v^*\!\left(
        \mathrm df_v^*\alpha,
        \mathrm df_v^*\beta
    \right).
    \label{eq:pushed-forward-cometric-bilinear}
\end{align}
Since $\mathrm df_v$ is surjective on the regular region, its pullback
$\mathrm df_v^*$ is injective. Hence, every nonzero task covector has a
nonzero pullback. Since $g_v^*$ is an inner product on $T_v^*\V$, it
follows that $d_v$ is an inner product on $T_w^*\W$, equivalently, a
co-metric at $w=f(v)$.

The co-metric $d_v$ is induced by the internal state $v$ and generally
varies among different internal states in the same fiber
$\mathcal F_w$. Therefore, the assignment $v\mapsto d_v$ does not, in
general, define a co-metric field depending only on the task value
$w$.

Since $d_v$ is a co-metric, its associated musical map sends task
covectors to task tangent vectors. To derive this map explicitly,
observe that
\begin{align}
    d_v(\alpha,\beta)
    =
    g_v^*\!\left(
        \mathrm df_v^*\alpha,
        \mathrm df_v^*\beta
    \right)
=
    \left(\mathrm df_v^*\beta\right)
    \!\left(
        g_v^\sharp\!\left(\mathrm df_v^*\alpha\right)
    \right).
    \label{eq:pushed-forward-cometric-intermediate}
\end{align}
Applying~\eqref{eq:pullback-evaluation} with
$\xi=g_v^\sharp(\mathrm df_v^*\alpha)$ gives
\begin{align}
    d_v(\alpha,\beta)
    =
    \beta\!\left(
        \mathrm df_v\!\left(
            g_v^\sharp\!\left(\mathrm df_v^*\alpha\right)
        \right)
    \right).
    \label{eq:pushed-forward-cometric-evaluation}
\end{align}
This identity identifies the sharp map associated with $d_v$ as
\begin{align}
    d_v^\sharp
    :=
    \mathrm df_v
    \circ g_v^\sharp
    \circ \mathrm df_v^*
    :
    T_w^*\W\to T_w\W,
    \label{eq:pushed-forward-cometric}
\end{align}
so that
\begin{align}
    d_v(\alpha,\beta)
    =
    \beta\!\left(d_v^\sharp(\alpha)\right).
    \label{eq:task-cometric-sharp-characterization}
\end{align}
The three maps in~\eqref{eq:pushed-forward-cometric} have distinct
roles: $\mathrm df_v^*$ pulls a task covector back to the internal
cotangent space, $g_v^\sharp$ associates it with an internal tangent
vector, and $\mathrm df_v$ maps that tangent vector to the task tangent
space.

Since $d_v$ is positive definite, its associated sharp map
$d_v^\sharp$ is an isomorphism. The inverse map
\begin{align}
    \bar d_v^\flat
    :=
    \left(d_v^\sharp\right)^{-1}
    =
    \left(
        \mathrm df_v
        \circ g_v^\sharp
        \circ \mathrm df_v^*
    \right)^{-1}
    :
    T_w\W\to T_w^*\W
    \label{eq:induced-task-flat-map}
\end{align}
is the flat map associated with an inner product $\bar d_v$ on
$T_w\W$. This inner product is defined by
\begin{align}
    \bar d_v(\zeta,\chi)
    :=
    \bar d_v^\flat(\zeta)(\chi),
    \quad
    \zeta,\chi\in T_w\W.
    \label{eq:induced-task-inner-product}
\end{align}
Equivalently,
\begin{align}
    \bar d_v(\zeta,\chi)
    =
    d_v\!\left(
        \bar d_v^\flat(\zeta),
        \bar d_v^\flat(\chi)
    \right).
\end{align}
Thus, $d_v$ is the state-induced co-metric on $T_w^*\W$, whereas
$\bar d_v$ is its inverse inner product on $T_w\W$. Their associated
musical isomorphisms satisfy
\begin{align}
    \bar d_v^\flat
    =
    \left(d_v^\sharp\right)^{-1},
    \quad
    d_v^\sharp
    =
    \left(\bar d_v^\flat\right)^{-1}.
\end{align}
Both $d_v$ and $\bar d_v$ generally depend on the internal state $v$.

To obtain the familiar matrix representation, choose coordinate
charts
\[
    x:\mathcal U\subset\V\to\mathbb R^n,
    \quad
    y:\mathcal O\subset\W\to\mathbb R^m,
\]
with $v\in\mathcal U$ and $w=f(v)\in\mathcal O$. Let
\[
    F:=y\circ f\circ x^{-1}
\]
be the local representation of the task map. The matrix representing
$\mathrm df_v$ in these coordinates is
\begin{align}
    J_f^{x,y}(v)
    :=
    DF_{x(v)}
    =
    D\!\left(y\circ f\circ x^{-1}\right)_{x(v)}.
    \label{eq:task-jacobian}
\end{align}

Let $G_x(v)$ denote the matrix representing the inner product $g_v$
in the $x$-coordinate basis. The inverse matrix $G_x(v)^{-1}$
represents the sharp isomorphism $g_v^\sharp$ in the associated
coordinate and dual bases. It also represents the induced co-metric
$g_v^*$ in the dual basis.

Consequently, the matrix representing $d_v^\sharp$, or equivalently
the matrix representing the co-metric $d_v$ in the task-coordinate
dual basis, is
\begin{align}
    D_{x,y}(v)
    :=
    J_f^{x,y}(v)
    G_x(v)^{-1}
    J_f^{x,y}(v)^\top.
    \label{eq:D-def}
\end{align}
After fixing the coordinate charts, we write simply $J_f(v)$,
$G(v)$, and $D(v)$.

The distinction among these objects is important. The matrix $G(v)$
represents the inner product $g_v$ on $T_v\V$. In the associated
coordinate and dual bases, $G(v)^{-1}$ represents both the sharp
isomorphism
\begin{align}
    g_v^\sharp:T_v^*\V\to T_v\V
\end{align}
and the co-metric $g_v^*$ on $T_v^*\V$.

Likewise, in the selected task-coordinate and dual bases, $D(v)$
represents both the sharp isomorphism
\begin{align}
    d_v^\sharp:T_w^*\W\to T_w\W
\end{align}
and the state-induced co-metric $d_v$ on $T_w^*\W$. Its inverse
$D(v)^{-1}$ represents both the flat isomorphism
\begin{align}
    \bar d_v^\flat:T_w\W\to T_w^*\W
\end{align}
and the corresponding state-induced inner product $\bar d_v$ on
$T_w\W$.

In the selected internal coordinates, the unit ball determined by
$g_v$ is
\begin{align}
    \mathcal B_v
    :=
    \left\{
        \dot x\in\mathbb R^n:
        \dot x^\top G(v)\dot x\leq 1
    \right\}.
    \label{eq:internal-unit-ball}
\end{align}
Here, $\dot x$ is the coordinate representation of a tangent vector
in $T_v\V$. Its physical interpretation depends on the quantities
represented by $\V$.

Since $\mathrm df_v$ is surjective on the regular region, the image of
$\mathcal B_v$ under $\mathrm df_v$ is the unit ball determined by the
state-induced inner product $\bar d_v$ on $T_w\W$. In the selected task
coordinates, this image is the ellipsoid
\begin{align}
    \mathcal E_v
    :=
    \left\{
        \dot y\in\mathbb R^m:
        \dot y^\top D(v)^{-1}\dot y\leq 1
    \right\},
    \label{eq:task-velocity-ellipsoid}
\end{align}
where $\dot y$ is the coordinate representation of a tangent vector
in $T_w\W$.

Thus, $D(v)^{-1}$ is the matrix of the inner product whose unit ball
is $\mathcal E_v$. Equivalently, $D(v)$ is the matrix of the dual
co-metric. Relative to the Euclidean structure of the selected task
coordinates, the eigenvectors of $D(v)$ determine the principal
directions of $\mathcal E_v$, and its eigenvalues are the squared
semiaxis lengths.

\subsection{Task-space volume and determinant proxies}
\label{subsec:determinant-proxies}

The state-induced inner product $\bar d_v$ and a task-space metric
serve different purposes. The inner product $\bar d_v$ is determined
by the internal metric $g_v$ and the differential $\mathrm df_v$; its
unit ball is precisely the capability ellipsoid $\mathcal E_v$.
Therefore, $\bar d_v$ describes the shape of the capability available
at the internal state $v$.

However, an ellipsoid does not determine a nontrivial numerical measure
of its own size. Indeed, if $\mathcal E_v$ is measured using the volume
form induced by $\bar d_v$, then, as the unit ball of $\bar d_v$, its
volume is always
\begin{align}
    \operatorname{vol}_{\bar d_v}(\mathcal E_v)=\omega_m,
\end{align}
where $\omega_m$ is the Euclidean volume of the unit ball in
$\mathbb R^m$. This value is independent of $v$ and therefore cannot
quantify variations in task capability.

A numerical comparison of capability ellipsoids consequently requires
a reference volume measure that is specified independently of the
ellipsoid being evaluated. We introduce this reference geometry by
endowing $\W$ with a Riemannian metric $h$. At each $w\in\W$, the
inner product
\begin{align}
    h_w:T_w\W\times T_w\W\to\R
\end{align}
defines the flat isomorphism
\begin{align}
    h_w^\flat:T_w\W\to T_w^*\W,
    \quad
    h_w^\flat(\zeta):=h_w(\zeta,\cdot).
    \label{eq:task-metric-flat-map}
\end{align}
Unlike $\bar d_v$, which generally varies among internal states in the
same fiber $\mathcal F_w$, the reference metric $h_w$ depends only on
the task point $w$. It specifies how tangent directions and volumes in
$T_w\W$ are to be compared.

The state-induced co-metric $d_v$ can be compared with the reference
metric $h_w$ through the composition
\begin{align}
    A_v
    :=
    d_v^\sharp\circ h_w^\flat
    :
    T_w\W\to T_w\W,
    \quad
    w=f(v).
    \label{eq:metric-completed-endomorphism}
\end{align}
This composition first maps a task tangent vector to a covector using
the reference metric $h_w$ and then maps that covector back to a task
tangent vector using the state-induced co-metric $d_v$. Hence, $A_v$
compares the state-induced capability geometry with the independently
selected reference geometry.

Unlike $d_v^\sharp$, whose domain and codomain are different vector
spaces, $A_v$ is an endomorphism of $T_w\W$. Its determinant is
therefore independent of the basis used to represent it. In the
selected task-coordinate and dual bases, $h_w^\flat$ and
$d_v^\sharp$ are represented by $H(w)$ and $D(v)$, respectively.
Consequently, $A_v$ is represented by
\begin{align}
    D(v)H(w).
\end{align}

This construction gives the metric-completed manipulability
\begin{align}
    \rho_H(v)
    :=
    \sqrt{\det A_v}
    =
    \sqrt{
        \det\!\bigl(D(v)H(f(v))\bigr)
    }.
    \label{eq:rhoH-def}
\end{align}
The corresponding volume of the capability ellipsoid, measured using
the reference metric $h_w$, is
\begin{align}
    \operatorname{vol}_{h_w}(\mathcal E_v)
    =
    \omega_m
    \sqrt{\det H(w)}
    \sqrt{\det D(v)}
    =
    \omega_m\rho_H(v).
    \label{eq:metric-ellipsoid-volume}
\end{align}
Hence, up to the dimension-dependent constant $\omega_m$,
$\rho_H(v)$ is the volume of the capability ellipsoid
$\mathcal E_v$ relative to the reference task metric $h$.

The geometric interpretation is consistent with the Riemannian
coarea formula. The metric-completed manipulability is the normal
Jacobian
\begin{align}
    J_{f,g,h}(v)
    &=
    \sqrt{
        \det\!\left(
            \mathrm df_v
            \circ g_v^\sharp
            \circ\mathrm df_v^*
            \circ h_{f(v)}^\flat
        \right)
    }
    =
    \rho_H(v).
\end{align}
Accordingly, for a suitable integrable scalar field
$\varphi:\V\to\R$,
\begin{align}
    \int_{\V}
    \varphi(v)\rho_H(v)\,
    \mathrm{dvol}_g(v)
    &=
    \int_{\W}
    \left(
        \int_{\mathcal F_w}
        \varphi(v)\,
        \mathrm{dvol}_{g|_{\mathcal F_w}}(v)
    \right)
    \mathrm{dvol}_h(w).
\end{align}
The reference task metric supplies the task-volume measure that turns
the determinant proxy into an intrinsic absolute-capability scalar.
Fiber normalization instead removes this task-volume factor by
comparing each state only with states on the same fiber.

If no task metric is specified, the usual determinant expression in
the selected task chart $y$ is
\begin{align}
    \rho_y(v)
    :=
    \sqrt{\det D(v)}.
    \label{eq:rho-def}
\end{align}
We call $\rho_y$ the raw determinant proxy. In the Euclidean case
$G(v)=I$, it reduces to
\begin{align}
    \rho_y(v)
    =
    \sqrt{
        \det\!\bigl(
            J_f(v)J_f(v)^\top
        \bigr)
    },
\end{align}
which is the classical volume-based manipulability expression.

Once a task chart has been fixed, we omit the subscript $y$ and write
simply $\rho$. The following sections examine how the determinant proxy and the
metric-completed manipulability behave under changes of coordinates and
task metric, and which optimization problems preserve their
solutions.

Figure~\ref{fig:pushed-forward-cometric} summarizes the complete tangent--cotangent construction, 
from the internal geometry induced by $g_v$ to the state-induced task geometry 
and its comparison with the reference task metric $h_w$.

\begin{figure}[pos=t]
    \centering
    \begin{tikzpicture}[
        node distance=12mm and 18mm,
        >=Stealth,
        font=\small
    ]

    \tikzset{
        space/.style={
            draw=Navy,
            rounded corners=2mm,
            minimum width=31mm,
            minimum height=16mm,
            align=center,
            fill=Pale,
            inner sep=2mm
        },
        inputspace/.style={
            space,
            draw=Red,
            fill=PaleOrange
        },
        explanationbox/.style={
            draw=Navy,
            dashed,
            rounded corners=2mm,
            fill=PaleBlue,
            align=left,
            font=\scriptsize,
            inner sep=2.5mm
        },
        geometrybox/.style={
            draw=Mid,
            dashed,
            rounded corners=2mm,
            inner sep=3mm
        },
        forward/.style={
            ->,
            thick,
            Blue
        },
        pullback/.style={
            ->,
            thick,
            Orange
        },
        induced/.style={
            ->,
            thick,
            Teal
        },
        reference/.style={
            ->,
            thick,
            Navy
        },
        explanationline/.style={
            dashed,
            thin,
            Navy
        },
        endomap/.style={
            ->,
            very thick,
            Red
        },
        annotation/.style={
            align=center,
            font=\scriptsize,
            text=Ink
        }
    }


    \node[space] (TV) {
        $T_v\V$\\[-0.1em]
        \scriptsize internal tangent vectors\\[-0.1em]
        \scriptsize inner product $g_v$
    };

    \node[space,below=of TV] (TVstar) {
        $T_v^*\V$\\[-0.1em]
        \scriptsize internal covectors\\[-0.1em]
        \scriptsize co-metric $g_v^*$
    };

    \node[space,right=25mm of TV] (TW) {
        $T_w\W$\\[-0.1em]
        \scriptsize output of $A_v$\\[-0.1em]
        \scriptsize induced inner product $\bar d_v$
    };

    \node[space,below=of TW] (TWstar) {
        $T_w^*\W$\\[-0.1em]
        \scriptsize task covectors\\[-0.1em]
        \scriptsize induced co-metric $d_v$
    };

    \node[inputspace,right=31mm of TW] (TWin) {
        $T_w\W$\\[-0.1em]
        \scriptsize input of $A_v$
    };

    \node[
        annotation,
        above=2mm of TWin
    ] (rho) {
        $\displaystyle
        \rho_H(v)=\sqrt{\det A_v}$
    };

    \node[
        explanationbox,
        below=8mm of TWin
    ] (Href) {
        \textbf{Reference task geometry}\\[0.3em]
        $h_w$ is an independently selected\\
        inner product on $T_w\W$.\\[0.2em]
        It depends on the task point $w$,\\
        rather than on the internal state $v$.
    };


    \draw[induced]
        ([xshift=-4mm]TV.south)
        --
        node[left] {$g_v^\flat$}
        ([xshift=-4mm]TVstar.north);

    \draw[induced]
        ([xshift=4mm]TVstar.north)
        --
        node[right] {$g_v^\sharp$}
        ([xshift=4mm]TV.south);


    \draw[forward]
        (TV.east)
        --
        node[above] {$\mathrm df_v$}
        (TW.west);

    \draw[pullback]
        (TWstar.west)
        --
        node[below] {$\mathrm df_v^*$}
        (TVstar.east);


    \draw[induced]
        ([xshift=-5mm]TWstar.north)
        --
        node[left] {$d_v^\sharp$}
        ([xshift=-5mm]TW.south);

    \draw[induced]
        ([xshift=5mm]TW.south)
        --
        node[right] {$\bar d_v^\flat$}
        ([xshift=5mm]TWstar.north);


    \draw[reference]
        (TWin.south west)
        to[bend left=-10]
        node[
            above left,
            align=center
        ] {$h_w^\flat$}
        (TWstar.north east);

    \coordinate (hflatmid) at
        ($(TWin.south west)!0.52!(TWstar.north east)$);

    \draw[explanationline]
        (Href.west)
        to[bend left=20]
        (hflatmid.south);


    \draw[endomap]
        (TWin.west)
        --
        node[
            above,
            text=Red,
            align=center
        ] {
            $A_v
            =
            d_v^\sharp\circ h_w^\flat$
        }
        (TW.east);


    \node[
        geometrybox,
        fit=(TV)(TVstar),
        label={
            [font=\scriptsize,text=Navy]
            above:internal geometry at $v$
        }
    ] {};

    \node[
        geometrybox,
        fit=(TW)(TWstar),
        label={
            [font=\scriptsize,text=Navy]
            above:state-induced task geometry at $(v,w)$
        }
    ] {};

    \node[
        geometrybox,
        fit=(TWin)(Href)(rho),
        label={
            [font=\scriptsize,text=Red]
            below:reference geometry and comparison input
        }
    ] {};

    \end{tikzpicture}

    \caption{
    Geometric construction associated with the task map
    $f:\V\to\W$, with $w=f(v)$. The internal metric $g_v$ induces
    the musical isomorphisms $g_v^\flat$ and $g_v^\sharp$ and the
    co-metric $g_v^*$ on $T_v^*\V$. The differential
    $\mathrm df_v$ maps internal tangent vectors to task tangent
    vectors, whereas its pullback $\mathrm df_v^*$ maps task covectors
    to internal covectors. Comparing these pullbacks through $g_v^*$
    produces the state-induced task co-metric $d_v$, whose sharp map
    is
    $d_v^\sharp
    =
    \mathrm df_v\circ g_v^\sharp\circ\mathrm df_v^*$.
    Its inverse
    $\bar d_v^\flat=(d_v^\sharp)^{-1}$ is the flat map associated
    with the state-induced inner product $\bar d_v$ on $T_w\W$.
    The unit ball of $\bar d_v$ is the capability ellipsoid
    $\mathcal E_v$. The independently selected task metric $h_w$
    defines the reference geometry used to measure this ellipsoid.
    Starting from the input copy of $T_w\W$, its flat map
    $h_w^\flat$ produces a task covector, which $d_v^\sharp$ maps to
    the output copy of $T_w\W$. Their composition is the endomorphism
    $A_v=d_v^\sharp\circ h_w^\flat:T_w\W\to T_w\W$.
    Its determinant defines the metric-completed manipulability
    $\rho_H(v)=\sqrt{\det A_v}$. In the selected coordinate and dual
    bases, $g_v^\sharp$, $d_v^\sharp$, $\bar d_v^\flat$,
    $h_w^\flat$, and $A_v$ are represented by $G(v)^{-1}$, $D(v)$,
    $D(v)^{-1}$, $H(w)$, and $D(v)H(w)$, respectively.
    }
    \label{fig:pushed-forward-cometric}
\end{figure}
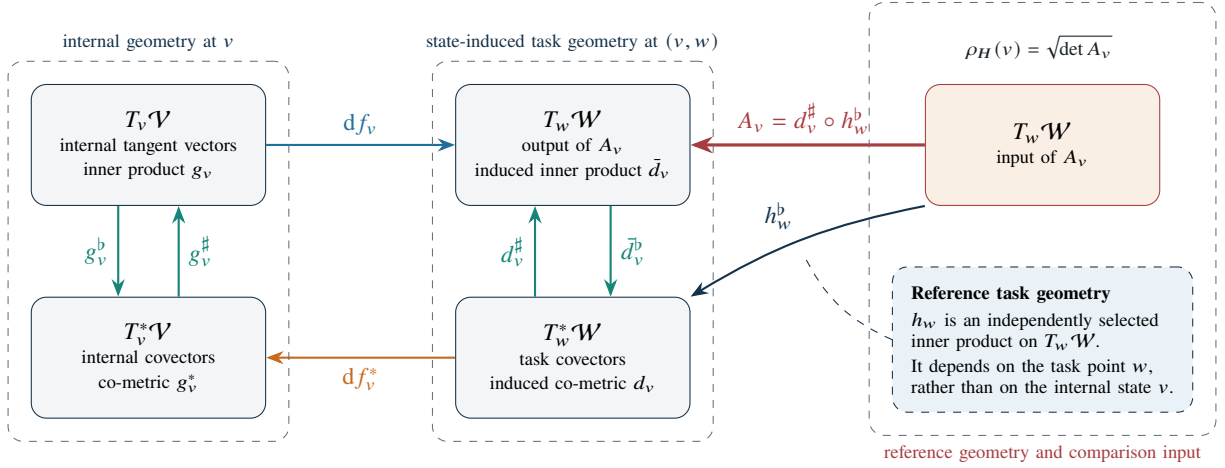

\section{Coordinate Transformations and Fiberwise Optimization Equivalence} \label{sec:main-results}

This section establishes the coordinate-transformation laws of the determinant proxy and metric-completed manipulability 
and derives their consequences for optimization on a fixed task fiber. We first distinguish cross-fiber comparison 
from fiber-restricted optimization, then prove the transformation lemmas, the fiberwise constancy proposition, and the resulting ordering-equivalence theorem.

\begin{asmpt}[Regularity for the optimization results]
\label{asm:regularity}
The analysis is restricted to the open regular region introduced in
Section~\ref{subsec:task-fibers}, where $f$ is smooth and
$\mathrm df_v$ is surjective. Results involving constrained Hessians
additionally require the restricted objective functions to be of class
$C^2$.
\end{asmpt}

\subsection{Global comparison versus fiber-restricted optimization}
\label{sec:global-vs-fiber}

\paragraph{Global use}
A global use compares $\rho$ or $\rho_H$ at configurations belonging
to different fibers, hence realizing different task values.
Examples include~\cite{PatelSobh2015,Vahrenkamp2012,Maric2019,XuZhanCao2018,Chiu1988,KazerounianWang1988,KleinBlaho1987,HuberWollherr2019,Zhang2020}.
Task-coordinate and task-metric choices can affect these comparisons.

\paragraph{Fiber-restricted use}
A fiber-restricted use optimizes $\rho$ or $\rho_H$ only on one
prescribed fiber $\Fw$. This includes self-motion, fixed-task
null-space redundancy resolution, co-contraction tuning, and
secondary-objective optimization under an already satisfied primary
task~\cite{Jin2017,Su2019,Franchi2026Coactivation, Franchi2026AeroPromptness}.
The fixed-fiber equivalence theorem below concerns this regime.

\subsection{Coordinate changes and determinant-based manipulability}
\label{subsec:coordinate-changes}

Although the pushed-forward co-metric $d_v$ is a geometric object,
its matrices $J_f(v)$, $G(v)$, and $D(v)$ depend on the coordinates
used to represent it. Configuration- and task-coordinate changes must
be distinguished because they affect these matrices in different
ways.

\subsubsection{Configuration-chart cancellation}

We first change coordinates on the configuration manifold while
keeping the task chart fixed. Let $x$ and $\widetilde x$ be two
overlapping charts on $\V$, and define the coordinate transition
\begin{align}
    \Phi_B:=x\circ\widetilde x^{-1},
    \quad
    T_B(v):=
    D\Phi_B\bigl(\widetilde x(v)\bigr).
    \label{eq:configuration-coordinate-transition}
\end{align}
The coordinate representations of a tangent vector in $T_v\V$ are
then related by
\begin{align}
    \dot x=T_B(v)\dot{\widetilde x}.
    \label{eq:configuration-tangent-transformation}
\end{align}

\begin{lemma}[Configuration-chart cancellation]
\label{lem:configuration-chart-cancellation}
Let the task chart remain fixed. Let $J_f(v)$, $G(v)$, and $D(v)$
denote the matrices obtained using the $x$ coordinates, and let
$\widehat J_f(v)$, $\widehat G(v)$, and $\widehat D(v)$ denote the
corresponding matrices in the $\widetilde x$ coordinates. Then
\begin{align}
    \widehat J_f(v)
    =
    J_f(v)T_B(v),
    \quad
    \widehat G(v)
    =
    T_B(v)^\top G(v)T_B(v),
    \label{eq:configuration-transformation-rules}
\end{align}
and
\begin{align}
    \widehat D(v)=D(v).
    \label{eq:configuration-coordinate-invariance}
\end{align}
\end{lemma}

\begin{proof}
The transformation rule for $\widehat J_f(v)$ follows from the chain
rule, whereas the transformation rule for $\widehat G(v)$ is the
coordinate transformation law of the inner product $g_v$. Therefore,
\begin{align}
    \widehat D(v)
    &=
    \widehat J_f(v)
    \widehat G(v)^{-1}
    \widehat J_f(v)^\top
    =
    J_f(v)T_B(v)
    \left(
        T_B(v)^{-1}
        G(v)^{-1}
        T_B(v)^{-\top}
    \right)
    T_B(v)^\top J_f(v)^\top
    \nonumber\\
    &=
    J_f(v)G(v)^{-1}J_f(v)^\top
    =
    D(v).
\end{align}
\end{proof}

Thus, the coordinate dependencies of $J_f$ and $G$ cancel exactly in
the combination $J_fG^{-1}J_f^\top$. Since the task chart is
unchanged, $D(v)$ and $\widehat D(v)$ are expressed in the same
task-coordinate and dual bases and coincide pointwise. This equality
is the matrix manifestation of the fact that $d_v$ and
$d_v^\sharp$ are geometric objects independent of the coordinates
chosen on $\V$.

\subsubsection{Task-chart covariance and determinant scaling}

We next keep the configuration chart fixed and change the coordinates
on the task manifold. Let $\psi$ and $\widetilde\psi$ be two
overlapping task charts, and write
\begin{align}
    y=\psi(w),
    \quad
    \widetilde y=\widetilde\psi(w).
\end{align}
Define the task-coordinate transition and its Jacobian by
\begin{align}
    \Phi_C:=\widetilde\psi\circ\psi^{-1},
    \quad
    T_C(w):=
    D\Phi_C\bigl(\psi(w)\bigr).
    \label{eq:task-coordinate-transition}
\end{align}
The coordinate representations of a tangent vector in $T_w\W$ are
related by
\begin{align}
    \dot{\widetilde y}=T_C(w)\dot y.
    \label{eq:task-tangent-transformation}
\end{align}

\begin{lemma}[Task-chart covariance and determinant scaling]
\label{lem:task-chart-covariance}
Let the configuration chart remain fixed. Let $J_f(v)$ and $D(v)$
denote the matrices obtained using the task chart $y$, and let
$\widetilde J_f(v)$ and $\widetilde D(v)$ denote the corresponding
matrices in the task chart $\widetilde y$. Then
\begin{align}
    \widetilde J_f(v)
    =
    T_C(f(v))J_f(v),
    \label{eq:task-jacobian-transformation}
\end{align}
and
\begin{align}
    \widetilde D(v)
    =
    T_C(f(v))
    D(v)
    T_C(f(v))^\top.
    \label{eq:task-coordinate-D-transformation}
\end{align}
Consequently, the determinant proxy transforms according to
\begin{align}
    \widetilde\rho(v)
    =
    \left\lvert
        \det T_C(f(v))
    \right\rvert
    \rho(v).
    \label{eq:raw-proxy-coordinate-transformation}
\end{align}
\end{lemma}

\begin{proof}
Equation~\eqref{eq:task-jacobian-transformation} follows from the
chain rule. Since the configuration chart remains fixed, the matrix
$G(v)$ is unchanged. Hence,
\begin{align}
    \widetilde D(v)
    &=
    \widetilde J_f(v)
    G(v)^{-1}
    \widetilde J_f(v)^\top
    =
    T_C(f(v))
    J_f(v)G(v)^{-1}J_f(v)^\top
    T_C(f(v))^\top
    \nonumber\\
    &=
    T_C(f(v))
    D(v)
    T_C(f(v))^\top,
\end{align}
which proves~\eqref{eq:task-coordinate-D-transformation}. Taking
determinants gives
\begin{align}
    \det\widetilde D(v)
    =
    \det T_C(f(v))^2
    \det D(v).
\end{align}
Taking the nonnegative square root proves
\eqref{eq:raw-proxy-coordinate-transformation}.
\end{proof}

Equation~\eqref{eq:task-coordinate-D-transformation} is the
transformation law of a contravariant two-tensor. Thus, a change of
task chart modifies the matrix representing the state-induced
co-metric $d_v$, but it does not modify the underlying geometric
object. By contrast, the determinant proxy acquires the
task-chart-dependent factor in
\eqref{eq:raw-proxy-coordinate-transformation} and therefore does not
define a scalar under task-coordinate changes. The absolute value
accounts for orientation-reversing chart transitions.

\subsubsection{Chart invariance through task-metric completion}

We now fix a Riemannian metric $h$ on $\W$. This metric provides the
reference task geometry used to measure the state-induced capability
ellipsoid. Let $H(w)$ and $\widetilde H(w)$ denote the matrices
representing the same inner product $h_w$ in the task charts $y$ and
$\widetilde y$, respectively.

\begin{lemma}[Chart invariance of metric-completed manipulability]
\label{lem:metric-completed-chart-invariance}
The matrices representing $h_w$ in the two task charts satisfy
\begin{align}
    \widetilde H(w)
    =
    T_C(w)^{-\top}
    H(w)
    T_C(w)^{-1}.
    \label{eq:task-metric-coordinate-transformation}
\end{align}
Moreover, the corresponding matrices representing the endomorphism
$A_v=d_v^\sharp\circ h_{f(v)}^\flat$ are related by similarity:
\begin{align}
    \widetilde D(v)\widetilde H(f(v))
    =
    T_C(f(v))
    D(v)H(f(v))
    T_C(f(v))^{-1}.
    \label{eq:metric-completed-similarity}
\end{align}
Consequently,
\begin{align}
    \widetilde\rho_H(v)=\rho_H(v).
    \label{eq:metric-completed-coordinate-invariance}
\end{align}
\end{lemma}

\begin{proof}
Equation~\eqref{eq:task-metric-coordinate-transformation} is the
coordinate transformation law of the covariant tensor $h_w$.
Combining this relation with
Lemma
~\ref{lem:task-chart-covariance} gives
\begin{align}
    \widetilde D(v)\widetilde H(f(v))
    &=
    T_C(f(v))
    D(v)
    T_C(f(v))^\top
    T_C(f(v))^{-\top}
    H(f(v))
    T_C(f(v))^{-1}
    \nonumber\\
    &=
    T_C(f(v))
    D(v)H(f(v))
    T_C(f(v))^{-1}.
\end{align}
Thus, $D(v)H(f(v))$ and
$\widetilde D(v)\widetilde H(f(v))$ are similar matrices and have
the same determinant:
\begin{align}
    \det\!\left(
        \widetilde D(v)\widetilde H(f(v))
    \right)
    =
    \det\!\left(
        D(v)H(f(v))
    \right).
\end{align}
Taking the nonnegative square root yields
\eqref{eq:metric-completed-coordinate-invariance}.
\end{proof}

Therefore, once the task metric $h$ has been fixed, the
metric-completed manipulability is independent of the task chart.
This chart invariance does not imply independence from the choice of
task metric. Different task metrics define different reference
geometries and volume measures. The task metric is part of the model,
whereas the task chart is only a choice of representation.

\subsection{Restriction to a fixed task fiber}
\label{subsec:fixed-fiber-scaling}

We now restrict the task-side transformation laws to a prescribed
fiber $\mathcal F_{w_0}$.
For later use, define
\begin{align}
    c_H(w)
    :=
    \sqrt{\det H(w)},
    \quad
    c_C(w)
    :=
    \left\lvert\det T_C(w)\right\rvert.
    \label{eq:task-side-factors}
\end{align}
Both functions are strictly positive wherever the task metric and
the coordinate transition are defined.

\begin{proposition}[Fiberwise constancy of task-side factors]
\label{prop:task-side-scaling}

Fix a regular task value $w_0\in f(\V)$ contained in the relevant
task charts. Then
\begin{align}
    \left.\rho_H\right|_{\mathcal F_{w_0}}
    =
    c_H(w_0)
    \left.\rho\right|_{\mathcal F_{w_0}},
    \quad
    \left.\widetilde\rho\right|_{\mathcal F_{w_0}}
    =
    c_C(w_0)
    \left.\rho\right|_{\mathcal F_{w_0}}.
    \label{eq:fiberwise-factorization}
\end{align}
Hence, both the metric-completed manipulability and the
coordinate-transformed determinant proxy are positive constant
rescalings of the determinant proxy on the entire fiber.
\end{proposition}

\begin{proof}
If $v\in\mathcal F_{w_0}$, then $f(v)=w_0$. Substitution into the
pointwise identities
\[
    \rho_H(v)=c_H(f(v))\rho(v),
    \quad
    \widetilde\rho(v)=c_C(f(v))\rho(v)
\]
gives~\eqref{eq:fiberwise-factorization}.
\end{proof}

Table~\ref{tab:coordinate-transformation-summary} collects the three
transformation results. To keep the formulas compact, the arguments
are suppressed according to
$T_B=T_B(v)$, $T_C=T_C(f(v))$, $D=D(v)$, and $H=H(f(v))$.

\begin{table}[pos=t]
    \centering
    \caption{
    Coordinate behavior of the manipulability constructions.
    In the matrix column, the arguments are suppressed according to
    $T_B=T_B(v)$, $T_C=T_C(f(v))$, $D=D(v)$, and
    $H=H(f(v))$.
    }
    \label{tab:coordinate-transformation-summary}

    \renewcommand{\arraystretch}{0}
    \setlength{\tabcolsep}{5pt}

    \begin{tabularx}{\linewidth}{
        @{}
        >{\raggedright\arraybackslash}p{0.20\linewidth}
        >{\centering\arraybackslash}p{0.25\linewidth}
        >{\centering\arraybackslash}X
        @{}
    }
        \toprule

        \textbf{Change}
        &
        \textbf{Matrix relations}
        &
        \textbf{Manipulability identities}
        \\

        \midrule

        \vspace{0pt}
        \textbf{Configuration chart}
        \newline
        {\footnotesize Exact cancellation}
        &
        \vspace{0pt}
        \(
        \begin{aligned}[t]
            \widehat J_f
            =
            J_fT_B&,
            \quad
            \widehat G
            =
            T_B^\top GT_B,
            \\
            \widehat D
            &=
            D
        \end{aligned}
        \)
        &
        \vspace{0pt}
        \(
        \begin{aligned}[t]
            \widehat\rho(v)
            &=
            \rho(v),
            \\
            \widehat\rho_H(v)
            &=
            \rho_H(v)
        \end{aligned}
        \)
        \\

        \midrule

        \vspace{0pt}
        \textbf{Task chart}
        \newline
        {\footnotesize Covariance and scaling}
        &
        \vspace{0pt}
        \(
        \begin{aligned}[t]
            \widetilde J_f
            &=
            T_CJ_f,
            \\
            \widetilde D
            &=
            T_CDT_C^\top
        \end{aligned}
        \)
        &
        \vspace{0pt}
        \(
        \begin{aligned}[t]
            \widetilde\rho(v)
            &=
            \left\lvert
                \det T_C(f(v))
            \right\rvert
            \rho(v)
        \end{aligned}
        \)
        \\

        \midrule

        \vspace{0pt}
        \textbf{Task chart and reference metric}
        \newline
        {\footnotesize Similarity and invariance}
        &
        \vspace{0pt}
        \(
        \begin{aligned}[t]
            \widetilde H
            &=
            T_C^{-\top}HT_C^{-1},
            \\
            \widetilde D\,\widetilde H
            &=
            T_C(DH)T_C^{-1}
        \end{aligned}
        \)
        &
        \vspace{0pt}
        \(
        \begin{aligned}[t]
            \rho_H(v)
            &=
            \sqrt{\det H(f(v))}\,\rho(v),
            \\
            \widetilde\rho_H(v)
            &=
            \sqrt{\det\widetilde H(f(v))}
            \,\widetilde\rho(v)
            \\
            &=
            \sqrt{\det H(f(v))}\,\rho(v)
            =
            \rho_H(v)
        \end{aligned}
        \)
        \\

        \bottomrule
    \end{tabularx}
\end{table}

\subsection{Fiberwise optimization equivalence}

Proposition~\ref{prop:task-side-scaling} converts the pointwise
transformation laws established above into positive constant scaling
relations on each prescribed fiber. The resulting equivalence is
stronger than equality of the optimizer sets: the objectives induce
the same complete ordering on the fiber.

\begin{theorem}[Complete fiberwise ordering equivalence]
\label{thm:main}
Fix a regular task value $w_0\in f(\V)$. Let $\rho$ be the determinant
proxy in one task chart, let $\widetilde\rho$ be the determinant proxy
in any overlapping task chart, and let $\rho_H$ be the
metric-completed manipulability associated with any Riemannian task
metric $h$.

Then, for every $u,v\in\mathcal F_{w_0}$,
\begin{align}
\begin{aligned}
    \rho(u)\leq\rho(v)
    \quad\Longleftrightarrow\quad
    \widetilde\rho(u)\leq\widetilde\rho(v)
    \quad\Longleftrightarrow\quad
    \rho_H(u)\leq\rho_H(v).
\end{aligned}
\label{eq:fiberwise-ordering-equivalence}
\end{align}
Thus, the three objectives induce the same complete ordering on
$\mathcal F_{w_0}$. In particular,
\begin{align}
    \argmax_{v\in\mathcal F_{w_0}}\rho(v)
    &=
    \argmax_{v\in\mathcal F_{w_0}}\widetilde\rho(v)
    =
    \argmax_{v\in\mathcal F_{w_0}}\rho_H(v),
    \label{eq:fiberwise-argmax-equivalence}
    \\
    \argmin_{v\in\mathcal F_{w_0}}\rho(v)
    &=
    \argmin_{v\in\mathcal F_{w_0}}\widetilde\rho(v)
    =
    \argmin_{v\in\mathcal F_{w_0}}\rho_H(v).
    \label{eq:fiberwise-argmin-equivalence}
\end{align}
These conclusions are unchanged by any consistent change of
configuration chart.
\end{theorem}

\begin{proof}
By Lemma~\ref{lem:configuration-chart-cancellation}, a
consistent change of configuration chart leaves $D(v)$ unchanged
pointwise. Therefore, it also leaves the determinant proxy and the
metric-completed manipulability unchanged pointwise.

By Proposition~\ref{prop:task-side-scaling}, the restrictions of
$\widetilde\rho$ and $\rho_H$ to $\mathcal F_{w_0}$ satisfy
\begin{align}
    \left.\widetilde\rho\right|_{\mathcal F_{w_0}}
    =
    c_C(w_0)
    \left.\rho\right|_{\mathcal F_{w_0}},
    \quad
    \left.\rho_H\right|_{\mathcal F_{w_0}}
    =
    c_H(w_0)
    \left.\rho\right|_{\mathcal F_{w_0}},
\end{align}
where $c_C(w_0)>0$ and $c_H(w_0)>0$. Multiplication by a positive
constant preserves the complete ordering of the objective values on
the fiber. This proves
\eqref{eq:fiberwise-ordering-equivalence}, from which
\eqref{eq:fiberwise-argmax-equivalence} and
\eqref{eq:fiberwise-argmin-equivalence} follow.

Finally,
Lemma~\ref{lem:metric-completed-chart-invariance} ensures that
$\rho_H$ is independent of the task chart used to represent it.
\end{proof}

The theorem establishes more than equality of the optimizer sets:
the determinant proxy, its representation in any overlapping task
chart, and every metric-completed manipulability induce exactly the
same ordering on a prescribed regular fiber. Therefore, the
determinant proxy is an exact objective for fixed-task redundancy
optimization. Its numerical values do not, however, define an
intrinsic ordering across different fibers because the positive
factors $c_C(w)$ and $c_H(w)$ generally vary with the task value.

\subsection{Fiberwise differential and local equivalence}

For completeness, we record the differential consequences of the
positive scaling relations in
Proposition~\ref{prop:task-side-scaling}. Let
\begin{align}
    g_v^{\mathcal F}
    :=
    \left.g_v\right|_{
        T_v\mathcal F_{w_0}\times T_v\mathcal F_{w_0}
    }
\end{align}
denote the inner product induced by $g_v$ on the tangent space of the
fiber, and let
\begin{align}
    \Pi_v:
    T_v\V\to T_v\mathcal F_{w_0}=V_v
\end{align}
be the $g_v$-orthogonal projection. For every smooth scalar field
$J$ defined near $\mathcal F_{w_0}$,
\begin{align}
    \operatorname{grad}_{g^{\mathcal F}}
    \left(
        \left.J\right|_{\mathcal F_{w_0}}
    \right)(v)
    =
    \Pi_v\operatorname{grad}_g J(v).
    \label{eq:restricted-gradient-projection}
\end{align}

\begin{corollary}[First-order fiberwise equivalence]
\label{cor:restricted-first-order}
Let $r\in\{\widetilde\rho,\rho_H\}$, and let $c_r(w_0)$ denote the
corresponding positive factor $c_C(w_0)$ or $c_H(w_0)$. For every
$v\in\mathcal F_{w_0}$,
\begin{align}
    \mathrm d\!\left(
        \left.r\right|_{\mathcal F_{w_0}}
    \right)_v
    &=
    c_r(w_0)
    \mathrm d\!\left(
        \left.\rho\right|_{\mathcal F_{w_0}}
    \right)_v,
    \label{eq:restricted-differential-scaling}
    \\
    \operatorname{grad}_{g^{\mathcal F}}
    \left(
        \left.r\right|_{\mathcal F_{w_0}}
    \right)(v)
    &=
    c_r(w_0)
    \operatorname{grad}_{g^{\mathcal F}}
    \left(
        \left.\rho\right|_{\mathcal F_{w_0}}
    \right)(v).
    \label{eq:restricted-gradient-scaling}
\end{align}
Consequently, $\rho$, $\widetilde\rho$, and $\rho_H$ have the same
constrained critical points on $\mathcal F_{w_0}$. Wherever their
restricted gradients are nonzero, the gradients define the same
oriented direction along the fiber.
\end{corollary}

\begin{proof}
Proposition~\ref{prop:task-side-scaling} gives
\begin{align}
    \left.r\right|_{\mathcal F_{w_0}}
    =
    c_r(w_0)
    \left.\rho\right|_{\mathcal F_{w_0}},
\end{align}
where $c_r(w_0)$ is constant and positive. Differentiation on
$\mathcal F_{w_0}$ proves
\eqref{eq:restricted-differential-scaling}. Taking the Riemannian dual
with respect to the induced inner product $g_v^{\mathcal F}$ proves
\eqref{eq:restricted-gradient-scaling}. Positivity of $c_r(w_0)$ then
preserves both the zeros and the orientation of the restricted
gradients.
\end{proof}

\begin{proposition}[Constrained Hessian scaling]
\label{prop:hessian}
Suppose that the restrictions of the objectives to
$\mathcal F_{w_0}$ are of class $C^2$. For every
$r\in\{\widetilde\rho,\rho_H\}$ and $v\in\mathcal F_{w_0}$,
\begin{align}
    \operatorname{Hess}_{g^{\mathcal F}}
    \left(
        \left.r\right|_{\mathcal F_{w_0}}
    \right)_v
    =
    c_r(w_0)
    \operatorname{Hess}_{g^{\mathcal F}}
    \left(
        \left.\rho\right|_{\mathcal F_{w_0}}
    \right)_v.
    \label{eq:restricted-hessian-scaling}
\end{align}
\end{proposition}

\begin{proof}
The restricted objectives differ by the constant factor $c_r(w_0)$,
and the covariant Hessian is linear under multiplication by a
constant.
\end{proof}

\begin{corollary}[Preservation of local optimality type]
\label{cor:secondorder}
A point $v^\star\in\mathcal F_{w_0}$ is a strict or non-strict local
maximum, local minimum, or saddle point of
$\left.\rho\right|_{\mathcal F_{w_0}}$ if and only if it has the same
classification for
$\left.\widetilde\rho\right|_{\mathcal F_{w_0}}$ and
$\left.\rho_H\right|_{\mathcal F_{w_0}}$.
\end{corollary}

\begin{proof}
For $r\in\{\widetilde\rho,\rho_H\}$ and every
$u,v^\star\in\mathcal F_{w_0}$,
\begin{align}
    r(u)-r(v^\star)
    =
    c_r(w_0)
    \bigl(
        \rho(u)-\rho(v^\star)
    \bigr).
\end{align}
Since $c_r(w_0)>0$, the sign of the objective difference is preserved
in every neighborhood of $v^\star$ within the fiber. The conclusion
therefore also applies to degenerate critical points, whose local type
need not be determined by the Hessian.
\end{proof}

\section{Fiber-normalized manipulability}
\label{sec:fiber-normalized}

We use fixed-fiber equivalence to construct a dimensionless index for
comparisons across task values. Each configuration is compared with
the best capability attainable on its own fiber.

\subsection{Definition and well-posedness}

\begin{asmpt}[Finite fiberwise reference]
\label{asm:fiberwise-reference}
For every task value under consideration, the restriction of
$\rho_H$ to $\mathcal F_w$ has a finite positive supremum. Whenever a
result characterizes fiberwise maximizers, this supremum is additionally
assumed to be attained.
\end{asmpt}

For $w\in f(\V)$, define the fiberwise reference value
\begin{align}
    \rho_H^\star(w)
    :=
    \sup_{u\in\mathcal F_w}\rho_H(u).
    \label{eq:rhoH-star}
\end{align}
On the regular region, $D(u)$ and $H(w)$ are positive definite for
every $u\in\mathcal F_w$, and therefore $\rho_H(u)>0$. Hence,
$\rho_H^\star(w)>0$ whenever the supremum is finite.

\begin{definition}[Fiber-normalized manipulability]
\label{def:fiber-normalized}
Under Assumption~\ref{asm:fiberwise-reference}, the
\emph{fiber-normalized manipulability} and its associated loss are
\begin{align}
    \mu_{\mathcal F}(v):=
    \frac{\rho_H(v)}{\rho_H^\star(f(v))},
    \quad
    \ell_{\mathcal F}(v):=1-\mu_{\mathcal F}(v).
    \label{eq:fiber-normalized}
\end{align}
\end{definition}
The denominator is evaluated on the fiber through $v$, rather than on
the whole configuration manifold. Thus, $\mu_{\mathcal F}(v)$ is the
fraction of the best capability available at the current task value
that the configuration attains. If the supremum is not attained, the
invariance results below remain valid, although
$\mu_{\mathcal F}=1$ need not be attained on the fiber.

\begin{proposition}[Range and fiberwise optimality]
\label{prop:fiber-normalized-range}
Under Assumption~\ref{asm:fiberwise-reference},
\begin{align}
    0<\mu_{\mathcal F}(v)\leq1,
    \quad
    0\leq\ell_{\mathcal F}(v)<1.
    \label{eq:mu-range}
\end{align}
Moreover, if the supremum on the fiber through $v$ is attained, then
\begin{align}
    \mu_{\mathcal F}(v)=1
    \quad\Longleftrightarrow\quad
    \ell_{\mathcal F}(v)=0
    \quad\Longleftrightarrow\quad
    v\in\argmax_{u\in\mathcal F_{f(v)}}\rho_H(u).
    \label{eq:mu-max-characterization}
\end{align}
\end{proposition}
\begin{proof}
The bounds follow by dividing
$0<\rho_H(v)\leq\rho_H^\star(f(v))$ by the positive denominator.
Equality holds precisely when $v$ attains the fiberwise supremum;
the loss statements follow from
$\ell_{\mathcal F}=1-\mu_{\mathcal F}$.
\end{proof}

\subsection{Independence of the task metric}

\begin{theorem}[Independence of the task metric]
\label{thm:task-metric-independence}
Let $h$ and $\widehat h$ be arbitrary Riemannian task metrics, and
let $\rho_H$ and $\rho_{\widehat H}$ be the corresponding
metric-completed manipulability measures defined by~\eqref{eq:rhoH-def}.
Under Assumption~\ref{asm:fiberwise-reference},
\begin{align}
    \frac{\rho_H(v)}
    {\displaystyle\sup_{u\in\mathcal F_{f(v)}}\rho_H(u)}
    =
    \frac{\rho_{\widehat H}(v)}
    {\displaystyle\sup_{u\in\mathcal F_{f(v)}}\rho_{\widehat H}(u)}
    \quad \forall v\in\V.
    \label{eq:metric-independent-mu}
\end{align}
Hence, $\mu_{\mathcal F}$ is independent of the task metric.
\end{theorem}

\begin{proof}
Fix $v$ and write $w=f(v)$. In a task chart about $w$,
Proposition~\ref{prop:task-side-scaling} gives
$\rho_H(u)=c_H(w)\rho(u)$ for every $u\in\Fw$.
Since $c_H(w)>0$ is constant on this fiber,
\begin{align}
    \mu_{\mathcal F}(v)
    =
    \frac{c_H(w)\rho(v)}
    {\displaystyle c_H(w)\sup_{u\in\Fw}\rho(u)}
    =
    \frac{\rho(v)}
    {\displaystyle\sup_{u\in\mathcal F_{f(v)}}\rho(u)}.
    \label{eq:mu-unweighted}
\end{align}
The same cancellation applies to $\widehat h$, yielding the same
right-hand side and proving~\eqref{eq:metric-independent-mu}.
\end{proof}

Equation~\eqref{eq:mu-unweighted} also provides a way to compute the
normalized index directly from the determinant proxy in any task chart.

\subsection{Coordinate invariance}

\begin{theorem}[Coordinate invariance]
\label{thm:mu-coordinate-invariance}
The fiber-normalized manipulability $\mu_{\mathcal F}$ is invariant
under coordinate changes on both $\V$ and $\W$.
\end{theorem}

\begin{proof}
For a fixed task metric, Lemma~\ref{lem:configuration-chart-cancellation} and
\eqref{eq:metric-completed-coordinate-invariance} show that $\rho_H$
has the same value at corresponding physical configurations.
Coordinate changes also leave the physical fiber unchanged.
Consequently, both $\rho_H(v)$ and its supremum over
$\mathcal F_{f(v)}$ are unchanged, and so is their quotient.
\end{proof}

Together, the two theorems identify $\mu_{\mathcal F}$ as an
intrinsic, dimensionless scalar determined by the task map $f$ and
the internal Riemannian metric $g$. Its independence from the task
metric $h$ does not imply independence from $f$ and $g$: changing
$f$ changes the task fibers, whereas changing $g$ can change the
relative ordering of internal states within each fiber.

\subsection{Cross-fiber comparison and optimization}

\begin{corollary}[Metric-independent cross-fiber optimization]
\label{cor:cross-fiber-optimization}
Let $\mathcal S\subseteq\V$ be a fixed feasible set, not necessarily
contained in one fiber. Provided that $\mu_{\mathcal F}$ is defined
on $\mathcal S$, the solution sets of
\begin{align}
    \argmax_{v\in\mathcal S}\mu_{\mathcal F}(v)
    \quad\text{and}\quad
    \argmin_{v\in\mathcal S}\ell_{\mathcal F}(v)
    \label{eq:cross-fiber-max}
\end{align}
are invariant under coordinate changes on $\V$ and $\W$ and under
changes of the task metric $h$.
\end{corollary}

\begin{proof}
The preceding theorems leave the objective values unchanged on the
same physical feasible set. Maximizing $\mu_{\mathcal F}$ is
equivalent to minimizing $1-\mu_{\mathcal F}$.
\end{proof}

\subsection{Cross-fiber trajectory optimization}
\label{subsec:trajectory-extension}

The preceding results concern static optimization, in which the
decision variable is an internal state selected from a feasible subset
of $\V$. We now consider trajectory optimization, in which the
decision variable is an internal-state trajectory
\begin{align}
    v:[0,T]\to\V.
\end{align}
The corresponding task trajectory is $f\circ v:[0,T]\to\W$. Depending
on the application, the task trajectory may be prescribed in advance
or selected by the optimizer.

The purpose of fiber normalization in this setting is to provide an
intrinsic state-dependent objective that remains meaningful while the
trajectory crosses different task fibers. The loss
$\ell_{\mathcal F}(v(t))$ evaluates the internal state at time $t$
relative to the best capability available on its instantaneous fiber
$\mathcal F_{f(v(t))}$. Physical requirements on the motion, such as
continuity, endpoint conditions, state constraints, speed bounds, or
motion penalties, can then be imposed independently through the
admissible trajectory class and the internal metric $g$.

\subsubsection{Intrinsic objectives for physically admissible trajectories}

Let $\mathcal A$ be a common physical admissible class of absolutely
continuous internal-state trajectories. The class may impose endpoint
conditions, state constraints, a prescribed task trajectory, and
intrinsic constraints on the tangent vector $\dot v$. Let
\begin{align}
    \Phi:[0,1]\times[0,\infty)\to\R
\end{align}
be continuous, and assume that the following integral is well defined
for every $v\in\mathcal A$:
\begin{align}
    \mathcal J_{\mathcal F,\Phi}[v]
    :=
    \int_0^T
    \Phi\!\left(
        \ell_{\mathcal F}(v(t)),
        g_{v(t)}\!\left(\dot v(t),\dot v(t)\right)
    \right)
    \,\dd t.
    \label{eq:intrinsic-cross-fiber-trajectory-objective}
\end{align}

\begin{theorem}[Intrinsic cross-fiber trajectory optimization]
\label{thm:intrinsic-cross-fiber-trajectory-optimization}
The value of
$\mathcal J_{\mathcal F,\Phi}[v]$ is invariant under coordinate
changes on $\V$ and $\W$ and independent of the reference task metric
$h$. Consequently, if the minimum is attained, the complete solution
set
\begin{align}
    \argmin_{v\in\mathcal A}
    \mathcal J_{\mathcal F,\Phi}[v]
    \label{eq:intrinsic-cross-fiber-trajectory-minimizers}
\end{align}
is coordinate invariant and independent of $h$.
\end{theorem}

\begin{proof}
The fiber-normalized loss $\ell_{\mathcal F}$ is invariant under
coordinate changes on $\V$ and $\W$ and independent of $h$. Moreover,
\begin{align}
    g_{v(t)}\!\left(\dot v(t),\dot v(t)\right)
\end{align}
is an intrinsic scalar and is therefore unchanged under coordinate
changes on $\V$. Hence, the integrand in
\eqref{eq:intrinsic-cross-fiber-trajectory-objective} has the same
value in every coordinate representation and for every choice of
$h$. Integration preserves this equality on every physical trajectory
in $\mathcal A$. The functional and, whenever the minimum is attained,
its complete minimizer set therefore have the stated properties.
\end{proof}

A useful instance of
\eqref{eq:intrinsic-cross-fiber-trajectory-objective} combines the
fiber-normalized loss with a quadratic penalty on internal motion:
\begin{align}
    \mathcal J_{\mathcal F,\lambda}[v]
    :=
    \int_0^T
    \left[
        \ell_{\mathcal F}(v(t))
        +
        \frac{\lambda}{2}
        g_{v(t)}\!\left(\dot v(t),\dot v(t)\right)
    \right]
    \,\dd t,
    \quad
    \lambda>0.
    \label{eq:fiber-normalized-regularized-trajectory-objective}
\end{align}
The first term favors internal states that attain a large fraction of
the capability available on their instantaneous fibers. The second
term penalizes rapid internal motion and couples the choices made at
different times. The parameter $\lambda$ determines the compromise
between relative redundancy utilization and motion regularity.

Alternatively, motion regularity can be imposed through the admissible
class rather than through the objective. For example, the intrinsic
speed-bounded class
\begin{align}
    \mathcal A_{\bar s}
    :=
    \left\{
        v\in\mathrm{AC}([0,T],\V)
        \ \middle|\
        \sqrt{
            g_{v(t)}\!\left(\dot v(t),\dot v(t)\right)
        }
        \leq\bar s
        \text{ for almost every }t
    \right\}
    \label{eq:intrinsic-speed-bounded-class}
\end{align}
rules out jumps and bounds the rate of internal motion. Minimizing
\begin{align}
    \mathcal J_{\mathcal F}[v]
    :=
    \int_0^T
    \ell_{\mathcal F}(v(t))
    \,\dd t
    \label{eq:normalized-path-fixed-param}
\end{align}
over $\mathcal A_{\bar s}$, possibly together with endpoint, state, or
task constraints, therefore defines an intrinsic optimization problem
over absolutely continuous trajectories crossing different fibers.

Theorem~\ref{thm:intrinsic-cross-fiber-trajectory-optimization}
establishes that these trajectory problems are intrinsically defined.
It does not, by itself, establish existence, uniqueness, or additional
smoothness of their minimizers. Such properties require further
assumptions on the admissible class, the fiber-normalized loss, and the
motion term.

\subsubsection{Prescribed task trajectories without temporal coupling}

We next identify the special case in which a prescribed-task
trajectory problem reduces to independent static optimizations on the
instantaneous fibers. Let
\begin{align}
    w:[0,T]\to\W
\end{align}
be a prescribed task trajectory. An admissible internal-state
trajectory must satisfy
\begin{align}
    f(v(t))=w(t),
    \quad
    \text{equivalently}
    \quad
    v(t)\in\mathcal F_{w(t)},
\end{align}
for almost every $t\in[0,T]$.

Consider first the uncoupled admissible class
\begin{align}
    \mathcal A_w^0
    :=
    \left\{
        v:[0,T]\to\V
        \ \middle|\
        \begin{array}{l}
            v \text{ is measurable},\\
            f(v(t))=w(t)
            \text{ for almost every }t
        \end{array}
    \right\}.
    \label{eq:uncoupled-admissible-class}
\end{align}
This class imposes the task constraint independently at each time. It
does not impose continuity, absolute continuity, endpoint conditions,
or bounds on $\dot v$. Therefore, the internal state selected at one
time does not restrict the selections available at other times.

For each task value, define the common fiberwise-optimal set
\begin{align}
    \mathcal M(w)
    :=
    \argmax_{u\in\mathcal F_w}\rho(u).
    \label{eq:common-fiberwise-optimal-set}
\end{align}
By Theorem~\ref{thm:main},
\begin{align}
\begin{aligned}
    \mathcal M(w)
    =
    \argmax_{u\in\mathcal F_w}\widetilde\rho(u)
    =
    \argmax_{u\in\mathcal F_w}\rho_H(u)
    =
    \argmin_{u\in\mathcal F_w}\ell_{\mathcal F}(u).
\end{aligned}
\label{eq:common-fiberwise-optimal-set-equivalence}
\end{align}

\begin{proposition}[Uncoupled prescribed-task equivalence]
\label{prop:prescribed-task-pointwise}
Let $w:[0,T]\to\W$ be a prescribed measurable task trajectory.
Assume that the fiberwise maxima defining $\mathcal M(w(t))$ are
attained for almost every $t$ and that there exists a measurable
selection $v^\star\in\mathcal A_w^0$ such that
\begin{align}
    v^\star(t)\in\mathcal M(w(t))
    \quad
    \text{for almost every }t.
    \label{eq:measurable-pointwise-minimizer}
\end{align}
Assume also that, for every $v\in\mathcal A_w^0$, the four integrands
below are measurable and integrable on $[0,T]$. Define
\begin{align}
    \mathcal J_\rho[v]
    &:=
    -\int_0^T \rho(v(t))\,\dd t,
    &
    \mathcal J_{\widetilde\rho}[v]
    &:=
    -\int_0^T \widetilde\rho(v(t))\,\dd t,
    \nonumber\\
    \mathcal J_{\rho_H}[v]
    &:=
    -\int_0^T \rho_H(v(t))\,\dd t,
    &
    \mathcal J_{\mathcal F}[v]
    &:=
    \int_0^T \ell_{\mathcal F}(v(t))\,\dd t.
    \label{eq:uncoupled-trajectory-objectives}
\end{align}
Then their complete minimizer sets on $\mathcal A_w^0$ coincide:
\begin{align}
\begin{aligned}
    \argmin_{v\in\mathcal A_w^0}\mathcal J_\rho[v]
    &=
    \argmin_{v\in\mathcal A_w^0}
    \mathcal J_{\widetilde\rho}[v]
    =
    \argmin_{v\in\mathcal A_w^0}
    \mathcal J_{\rho_H}[v]
    =
    \argmin_{v\in\mathcal A_w^0}
    \mathcal J_{\mathcal F}[v]
    \\
    &=
    \left\{
        v\in\mathcal A_w^0
        \ \middle|\
        v(t)\in\mathcal M(w(t))
        \text{ for almost every }t
    \right\}.
\end{aligned}
\label{eq:uncoupled-trajectory-equivalence}
\end{align}
\end{proposition}

\begin{proof}
For almost every $t$, every admissible trajectory satisfies
$v(t)\in\mathcal F_{w(t)}$. By
\eqref{eq:common-fiberwise-optimal-set-equivalence}, the four
instantaneous objectives have the common optimal set
$\mathcal M(w(t))$ on that fiber. Therefore, every measurable
selection satisfying
\eqref{eq:measurable-pointwise-minimizer} attains the pointwise
minimum of all four integrands almost everywhere and consequently
minimizes all four accumulated objectives.

Conversely, fix one of the four objectives and suppose that
$v\in\mathcal A_w^0$ is pointwise suboptimal for its integrand on a
measurable set $B\subseteq[0,T]$ of positive measure. Define the
measurable patched selection
\begin{align}
    \widehat v(t)
    :=
    \begin{cases}
        v^\star(t), & t\in B,\\
        v(t),       & t\notin B.
    \end{cases}
\end{align}
Since both $v$ and $v^\star$ satisfy the prescribed task constraint
almost everywhere, $\widehat v\in\mathcal A_w^0$. The integrand
evaluated at $\widehat v(t)$ is no larger than that evaluated at
$v(t)$ almost everywhere and is strictly smaller on $B$. Their
difference is measurable, integrable, nonnegative, and positive on a
set of positive measure. Hence, the corresponding accumulated
objective is strictly smaller at $\widehat v$ than at $v$.

Thus, a minimizer of any of the four objectives must belong to
$\mathcal M(w(t))$ almost everywhere. Since this instantaneous optimal
set is common to all four integrands, their complete minimizer sets
coincide and are given by
\eqref{eq:uncoupled-trajectory-equivalence}.
\end{proof}

Proposition~\ref{prop:prescribed-task-pointwise} identifies the
uncoupled limit of prescribed-task trajectory optimization. Since no
condition links different times, the problem reduces to independent
fiberwise optimizations, and all four objectives yield the same
measurable pointwise-optimal selections. Such selections need not be
continuous and may contain instantaneous jumps when the set
$\mathcal M(w)$ changes discontinuously or has disconnected branches.

\subsubsection{Temporally coupled cross-fiber trajectories}

Physical trajectory optimization generally requires more than
measurability. Continuity, absolute continuity, endpoint conditions,
intrinsic speed bounds, or motion penalties constrain how an internal
state selected on one fiber can be connected to states on neighboring
fibers. The resulting problem no longer decomposes into independent
instantaneous optimizations.

For a prescribed task trajectory, these constraints define a problem
of selecting a physically admissible lift
\begin{align}
    v:[0,T]\to\V,
    \quad
    f(v(t))=w(t),
\end{align}
that moves across the fibers $\mathcal F_{w(t)}$. For a free task
trajectory, both the internal-state trajectory and its image
$f\circ v$ are selected by the optimizer. In either case, the
fiber-normalized loss supplies an intrinsic measure of relative
redundancy utilization across the visited fibers, while the admissible
class or the motion term determines how those fibers can be traversed.

The positive task-dependent factors relating $\rho$,
$\widetilde\rho$, and $\rho_H$ are constant on each instantaneous
fiber but generally vary along a trajectory that crosses fibers.
When different times are coupled, these varying factors can change
the relative importance assigned to different portions of a candidate
trajectory. The unnormalized objectives can therefore select different
transitions or paths.

By contrast,
Theorem~\ref{thm:intrinsic-cross-fiber-trajectory-optimization}
shows that objectives constructed from $\ell_{\mathcal F}$ and
intrinsic motion quantities remain coordinate invariant and independent
of the reference task metric on any common physical admissible class.
This is the principal trajectory-level role of fiber normalization:
it enables meaningful optimization of physically admissible
cross-fiber trajectories without introducing task-chart-dependent or
task-metric-dependent weighting of the visited fibers.

The prescribed-task speed-bounded problem and the free-task planning
problem are illustrated in
Sections~\ref{sec:prescribed-task-example}
and~\ref{sec:free-task-example}, respectively.

Table~\ref{tab:fiber-normalized-optimization-summary} summarizes the successive uses of fiber normalization, from pointwise cross-fiber comparison to physically admissible trajectory optimization.

\begin{table}[pos=t]
    \centering
    \caption{
    Summary of the optimization problems enabled by fiber-normalized
    manipulability. The admissible sets are physical sets and are
    therefore kept fixed when coordinates or the reference task metric
    are changed.
    }
    \label{tab:fiber-normalized-optimization-summary}

    \renewcommand{\arraystretch}{0}
    \setlength{\tabcolsep}{5pt}

    \begin{tabularx}{\linewidth}{
        @{}
        >{\raggedright\arraybackslash}p{0.22\linewidth}
        >{\raggedright\arraybackslash}p{0.40\linewidth}
        >{\raggedright\arraybackslash}X
        @{}
    }
        \toprule

        \textbf{Problem}
        &
        \textbf{Construction}
        &
        \textbf{Result}
        \\

        \midrule

        \vspace{0pt}
        \textbf{Fiber-normalized value}
        \newline
        {\footnotesize Pointwise comparison}
        &
        \vspace{0pt}
        \(
        \begin{aligned}[t]
            \mu_{\mathcal F}(v)
            &=
            \frac{\rho_H(v)}
            {\displaystyle
             \sup_{u\in\mathcal F_{f(v)}}\rho_H(u)},
            \\
            \ell_{\mathcal F}(v)
            &=
            1-\mu_{\mathcal F}(v)
        \end{aligned}
        \)
        &
        \vspace{0pt}
        \(
        \begin{aligned}[t]
            0
            &<
            \mu_{\mathcal F}(v)
            \leq 1,
            \\
            0
            &\leq
            \ell_{\mathcal F}(v)
            <1
        \end{aligned}
        \)
        \newline
        {\footnotesize
        Coordinate invariant and independent of \(h\).}
        \\

        \midrule

        \vspace{0pt}
        \textbf{Static cross-fiber optimization}
        \newline
        {\footnotesize \(v\in\mathcal S\subseteq\V\)}
        &
        \vspace{0pt}
        \(
        \begin{aligned}[t]
            \max_{v\in\mathcal S}
            \mu_{\mathcal F}(v)
            \quad
            \Longleftrightarrow
            \quad
            \min_{v\in\mathcal S}
            \ell_{\mathcal F}(v)
        \end{aligned}
        \)
        &
        \vspace{0pt}
        {\footnotesize
        The complete optimizer set is coordinate invariant and
        independent of \(h\).}
        \\

        \midrule

        \vspace{0pt}
        \textbf{Intrinsic trajectory optimization}
        \newline
        {\footnotesize \(v\in\mathcal A\subset
        \mathrm{AC}([0,T],\V)\)}
        &
        \vspace{0pt}
        \(
        \begin{aligned}[t]
            \mathcal J_{\mathcal F,\Phi}[v]
            &:=
            \int_0^T
            \Phi\!\left(
                \ell_{\mathcal F}(v(t)),
                \|\dot v(t)\|_{g_{v(t)}}^2
            \right)
            \,\dd t
        \end{aligned}
        \)
        &
        \vspace{0pt}
        {\footnotesize
        The functional and its complete minimizer set are coordinate
        invariant and independent of \(h\).}
        \\

        \midrule

        \vspace{0pt}
        \textbf{Prescribed task, uncoupled}
        \newline
        {\footnotesize
        \(v\in\mathcal A_w^0\)}
        &
        \vspace{0pt}
        \(
        \begin{aligned}[t]
            f(v(t))
            &=
            w(t),
            \\
            v(t)
            &\in
            \mathcal M(w(t))
            \quad\text{a.e.}
        \end{aligned}
        \)
        &
        \vspace{0pt}
        \(
        \begin{aligned}[t]
            \argmin\mathcal J_\rho
            &=
            \argmin\mathcal J_{\widetilde\rho}
            =
            \argmin\mathcal J_{\rho_H}
            \\
            &=
            \argmin\mathcal J_{\mathcal F}
        \end{aligned}
        \)
        \newline
        {\footnotesize
        Pointwise selections may be discontinuous.}
        \\

        \midrule

        \vspace{0pt}
        \textbf{Cross-fiber motion with temporal coupling}
        \newline
        {\footnotesize Prescribed or free task}
        &
        \vspace{0pt}
        \(
        \begin{aligned}[t]
            v
            &\in
            \mathrm{AC}([0,T],\V),
            \\
            \|\dot v(t)\|_{g_{v(t)}}
            &\leq
            \bar s
            \quad\text{a.e.}
        \end{aligned}
        \)
        \newline
        {\footnotesize
        Endpoint, state, or task constraints may also be imposed.}
        &
        \vspace{0pt}
        {\footnotesize
        Unnormalized objectives may select different paths.
        Fiber-normalized objectives remain coordinate invariant and
        independent of \(h\).}
        \\

        \bottomrule
    \end{tabularx}
\end{table}

\section{Numerical case studies and robotics implications}
\label{sec:examples}

\subsection{Organization and common comparison protocol}
\label{sec:examples-roadmap}

The planar 2R manipulator and dual-rotor allocation model illustrate
three optimization settings:
\begin{enumerate}[label=(\roman*)]
    \item static fiberwise and global optimization for both systems
    (Sec.~\ref{sec:static-examples});
    \item allocation along a prescribed dual-rotor task trajectory,
    first without temporal coupling and then under a speed bound
    (Sec.~\ref{sec:prescribed-task-example});
    \item free-task planning between fixed configurations for both
    systems (Sec.~\ref{sec:free-task-example}).
\end{enumerate}
These experiments illustrate the theoretical results; they are not
intended as independent proofs.

Throughout, $c_C$ and $c_H$ are the task-side factors
in~\eqref{eq:task-side-factors}. For static optimization and landscape
visualization, define
\begin{align}
    L(v)=-\log\rho(v),\quad
    \widetilde L(v)=L(v)-\log c_C(f(v)),\quad
    L_H(v)=L(v)-\log c_H(f(v)).
    \label{eq:example-log-costs}
\end{align}
Since $-\log$ is strictly decreasing, minimizing $L$ is equivalent
to maximizing $\rho$. Theorem~\ref{thm:main} therefore gives the same
fiberwise minimizers for all three logarithmic costs.

Their fiber-optimum profiles are
\begin{align}
    L^\star(w):=\min_{u\in\Fw}L(u),\quad
    \widetilde L^\star(w):=\min_{u\in\Fw}\widetilde L(u),\quad
    L_H^\star(w):=\min_{u\in\Fw}L_H(u).
    \label{eq:common-fiber-optimum-profiles}
\end{align}
Proposition~\ref{prop:task-side-scaling} gives
\begin{align}
    \widetilde L^\star(w)=L^\star(w)-\log c_C(w),
    \quad
    L_H^\star(w)=L^\star(w)-\log c_H(w).
    \label{eq:common-profile-transformations}
\end{align}
These task-dependent shifts can change which fiber has the lowest
optimum value. The trajectory problems use the running costs
$-\rho$, $-\widetilde\rho$, $-\rho_H$, and $\ell_{\mathcal F}$
directly, not their logarithmic representations.

The nonconstant task-coordinate and task-metric factors used below
are deliberately chosen to make their cross-fiber effect visible.
Their particular analytical expressions are not required by the
theory; only smoothness, nonsingularity of the coordinate change, and
positive definiteness of the metric are used.

Changes of configuration coordinates are not plotted separately:
Lemma~\ref{lem:configuration-chart-cancellation} shows that the corresponding
transformations of $J_f$ and $G$ cancel exactly and leave $D$
unchanged pointwise.

Two graphical conventions are used below. In the static landscape
figures, red curves represent task fibers, yellow curves represent
the common fiberwise-optimal locus, and square, circular, and diamond
markers identify global selections associated with the determinant-
proxy, task-coordinate, and task-metric constructions, respectively.
In the trajectory-planning figures, the objective fields are rendered
in grayscale, thin gray curves denote task fibers, and gold curves
denote the common fiberwise-optimal locus. The optimized lifts are
drawn using the high-contrast colors and line styles identified in
the corresponding legends.

\subsection{Static fiberwise and global optimization}
\label{sec:static-examples}

\subsubsection{Planar 2R manipulator}
\label{sec:2r-static}

Consider a planar two-link manipulator with joint configuration
$q=(q_1,q_2)\in\mathbb T^2$ and link lengths
$\ell_1=1$ and $\ell_2=0.8$. We equip the configuration manifold with
the flat product metric, whose matrix in the joint coordinates is
$G(q)=I_2$. As a one-dimensional task, consider the horizontal
end-effector position
\begin{align}
    f(q)=x_E(q)
    =\ell_1\cos q_1+\ell_2\cos(q_1+q_2)=w.
\end{align}
The arm is therefore redundant with respect to this scalar task, with
one-dimensional regular fibers in its two-dimensional configuration
manifold.
Its Jacobian is
\begin{align}
    J_f(q)=
    \begin{bmatrix}
        -\ell_1\sin q_1-\ell_2\sin(q_1+q_2)
        &
        -\ell_2\sin(q_1+q_2)
    \end{bmatrix}.
\end{align}
Consequently, on the regular set $D(q)>0$,
\begin{align}
    D(q)=J_f(q)J_f(q)^\top,
    \quad
    \rho(q)=\sqrt{D(q)}=\lVert J_f(q)\rVert_2,
\end{align}
with $L$ given by~\eqref{eq:example-log-costs}.

For $w\in[-1.65,1.65]$, denote the fiberwise optimizer set by
\begin{align}
    \mathcal Q^\star(w):={\arg\min}_{q\in\mathcal F_w}L(q).
    \label{eq:2r-fiberwise-minimizers}
\end{align}
To illustrate task-side dependence, consider the smooth,
orientation-preserving task-coordinate change
\begin{align}
    \widetilde w=\phi(w):=\frac{e^{0.9w}-1}{0.9},
    \quad
    c_C(w)=\phi'(w)=e^{0.9w}>0.
\end{align}
The modified logarithmic cost is therefore
\begin{align}
    \widetilde L(q)=L(q)-0.9f(q).
\end{align}
We also consider the one-dimensional task-space Riemannian metric with matrix representation
\begin{align}
    H(w)=e^{-1.8w}>0.
\end{align}
Since $c_H(w)=\sqrt{H(w)}=e^{-0.9w}$, its logarithmic cost is
\begin{align}
    L_H(q)=L(q)+0.9f(q).
\end{align}

For the profiles defined in
\eqref{eq:common-fiber-optimum-profiles},
\eqref{eq:common-profile-transformations} specializes to
\begin{align}
    \widetilde L^\star(w)
    =
    L^\star(w)-0.9w,
    \quad
    L_H^\star(w)
    =
    L^\star(w)+0.9w.
    \label{eq:2r-fiber-profile-transformations}
\end{align}

Figure~\ref{fig:2R-landscapes} shows visibly different
configuration-space cost landscapes with a common fiberwise-optimal
locus. The red curves represent selected fixed-task fibers, while the
yellow branches form the locus
\begin{align}
    \bigcup_{w\in[-1.65,1.65]}
    \mathcal Q^\star(w).
\end{align}
Its coincidence across the three panels illustrates the fiberwise
ordering equivalence established by Theorem~\ref{thm:main}, despite
the different cross-fiber values of the costs.

The marked and annotated configurations in
Fig.~\ref{fig:2R-landscapes} are global minimizers of the cost shown
in the corresponding panel. The square identifies a global minimizer
of $L$, the circle identifies a global minimizer of
$\widetilde L$, and the diamond identifies a global minimizer of
$L_H$. The adjacent annotations report their joint coordinates.
Although all three marked configurations belong to the common
fiberwise-optimal locus, the task-dependent modifications can select
global minimizers on different fibers.

Figure~\ref{fig:2R-summary} examines the distinction between
fiberwise and global optimization more directly. To illustrate the
fiberwise comparison, fix $w_0=0.60$. The fiber
$\mathcal F_{w_0}$ can be parameterized by $q_1$ through the two
inverse-kinematics branches
\begin{align}
    q_2^\pm(q_1;w_0)
    =
    \pm\arccos\!\left(
        \frac{w_0-\cos q_1}{0.8}
    \right)-q_1,
\end{align}
where defined. Panel~(a) shows the restrictions of $L$,
$\widetilde L$, and $L_H$ to these branches. Since
$f(q)=w_0$ everywhere on the fiber, the three restricted costs differ
only by constant vertical translations and therefore have the same
minimizers.

Panel~(b) displays the complete common fiberwise-optimal locus and
locates on it the global selections obtained from the three costs.
Panel~(c) shows the corresponding fiber-optimum profiles
$L^\star$, $\widetilde L^\star$, and $L_H^\star$. The opposite
task-dependent shifts in
\eqref{eq:2r-fiber-profile-transformations} explain why the three
costs preserve the complete ordering on every prescribed fiber while
selecting different task values under global comparison.

\begin{figure}[pos=t]
    \centering
    \includegraphics[width=\linewidth]{%
        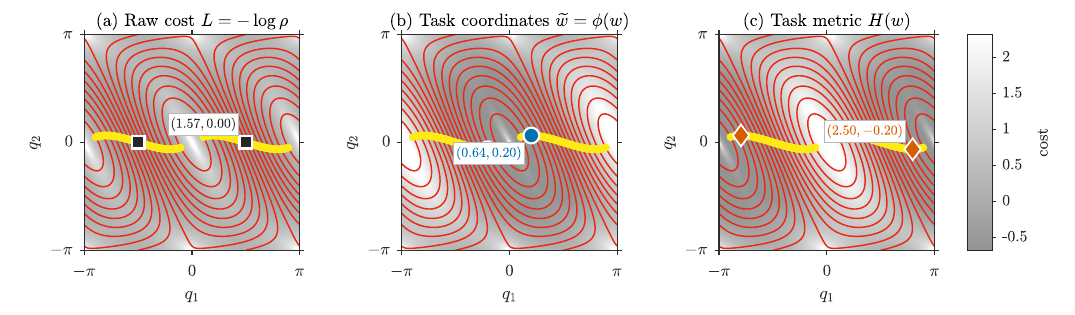}
    \caption{
    Configuration-space logarithmic-cost landscapes for the planar
    2R manipulator. The grayscale background represents the logarithmic cost, with darker
gray indicating lower displayed values and white indicating higher
displayed values. The same cost scale is used across all three panels.
(a) Determinant-proxy cost $L=-\log\rho$.
    (b) Cost $\widetilde L$ obtained after the nonlinear
    task-coordinate change. (c) Cost $L_H$ obtained using the
    reference task metric. Red curves show representative fixed-task
    fibers, and the common yellow branches form the fiberwise-optimal
    locus. The square, circle, and diamond identify global minimizers
    of $L$, $\widetilde L$, and $L_H$, respectively; the adjacent
    annotations report their joint coordinates. The three landscapes
    differ across fibers, whereas their minimizers coincide on every
    prescribed fiber.
    }
    \label{fig:2R-landscapes}
\end{figure}

    \begin{figure}[pos=t]
    \centering
    \includegraphics[width=\linewidth]{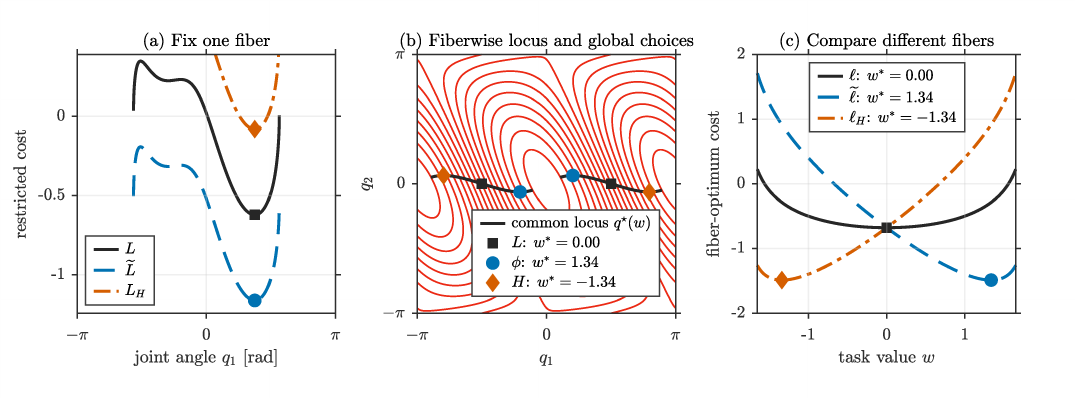}
    \caption{
    Fiberwise and global comparisons for the planar 2R manipulator.
    (a) Restrictions of $L$, $\widetilde L$, and $L_H$ to
    $\mathcal F_{w_0}$, with $w_0=0.60$, parameterized by $q_1$.
    The three curves differ only by constant vertical translations
    and have the same minimizers. (b) Common fiberwise-optimal locus,
    with the square, circle, and diamond identifying the global
    selections produced by $L$, $\widetilde L$, and $L_H$,
    respectively. (c) Fiber-optimum profiles
    $L^\star$, $\widetilde L^\star$, and $L_H^\star$.
    Task-side modifications preserve the complete ordering on each
    fixed fiber but can change the task value selected by a global
    comparison.
    }
    \label{fig:2R-summary}
\end{figure}

\subsubsection{Dual-rotor aerodynamic allocation}
\label{sec:dual-static}

Consider two collinear rotors with spin-rate coordinates
$v=(v_1,v_2)\in\mathbb R^2$ and scalar generalized-force task
\begin{align}
    f(v)=v_1\lvert v_1\rvert+0.7v_2\lvert v_2\rvert=w.
    \label{eq:dual-task-map}
\end{align}
For each motor--propeller unit, consider
\begin{align}
    m_i\dot v_i=-b_iv_i\lvert v_i\rvert+\tau_i,
    \quad
    \lvert\tau_i\rvert\leq\bar\tau_i,
\end{align}
with \(
    b=(0.10,0.10),
    \quad
    \bar\tau=(1,1),
    \quad
    m=(0.05,0.05)\).
The corresponding symmetric acceleration capacities (SAC)
are~\cite{Franchi2026AeroPromptness}
\begin{align}
    s_i(v_i)
    =\frac{\bar\tau_i-b_iv_i^2}{m_i}
    =20-2v_i^2,
    \quad i\in\{1,2\}.
    \label{eq:dual-sac}
\end{align}
Following the drag-aware aerodynamic manipulability construction
of~\cite{Franchi2026AeroPromptness}, we equip the open feasible
spin-rate space
\begin{align}
    \mathcal E=(-\sqrt{10},\sqrt{10})^2
\end{align}
with the SAC metric
\begin{align}
    G(v)=\operatorname{diag}\!\left(
        \frac{1}{s_1(v_1)^2},
        \frac{1}{s_2(v_2)^2}
    \right).
\end{align}
The metric coefficients diverge as the available symmetric
acceleration approaches zero at the actuator boundary. Since
\begin{align}
    J_f(v)=
    \begin{bmatrix}
        2\lvert v_1\rvert & 1.4\lvert v_2\rvert
    \end{bmatrix},
\end{align}
the matrix representing the state-induced task co-metric is the scalar
\begin{align}
    D(v)
    =
    J_f(v)G(v)^{-1}J_f(v)^\top
    =
    4v_1^2\left(20-2v_1^2\right)^2
    +1.96v_2^2\left(20-2v_2^2\right)^2.
    \label{eq:dual-D}
\end{align}
The corresponding determinant proxy $\rho$ is the drag-aware
aerodynamic manipulability (DAAM), with logarithmic cost given
by~\eqref{eq:example-log-costs}. The task map
in~\eqref{eq:dual-task-map} is of class $C^1$ and is a submersion on
$\mathcal E_{\mathrm{reg}}=\mathcal E\setminus\{0\}$. Consequently,
Proposition~\ref{prop:task-side-scaling} and
Theorem~\ref{thm:main} apply on this regular set. The constrained
Hessian result in Proposition~\ref{prop:hessian} and the local
classification in Corollary~\ref{cor:secondorder} apply away from the
coordinate axes, where the required second derivatives exist.

For $w\in[-10,10]$, define
\begin{align}
    \mathcal V^\star(w)
    :=
    \argmin_{v\in\mathcal F_w}L(v),
    \quad
    \mathcal F_w\subset\mathcal E.
\end{align}
Consider the smooth task-coordinate change
\begin{align}
    \widetilde w=\phi(w)
    :=w+1.35\tanh\!\left(\frac{w-5}{3}\right),
\end{align}
with
\begin{align}
    c_C(w)=\phi'(w)
    =1+0.45\left[
        1-\tanh^2\!\left(\frac{w-5}{3}\right)
    \right]>0.
\end{align}
The positivity of $c_C$ makes $\phi$ orientation preserving. The
transformation fixes $w=5$ and produces a localized coordinate
stretching around that value, where $c_C(5)=1.45$, while
$c_C(w)\to1$ away from the transition region. It therefore modifies
cross-fiber comparisons most strongly for task values near $w=5$.

We also introduce the scalar reference task metric with matrix
representation
\begin{align}
    H(w)=\left[
        1+0.60\exp\!\left(-\frac{(w+5)^2}{18}\right)
    \right]^2,
\end{align}
so that its volume-density factor is
\begin{align}
    c_H(w)=\sqrt{H(w)}
    =1+0.60\exp\!\left(-\frac{(w+5)^2}{18}\right)>0.
\end{align}
This choice produces a localized metric weighting centered at
$w=-5$, opposite to the coordinate stretching centered at $w=5$.
The costs and fiber-optimum profiles follow
\eqref{eq:example-log-costs} and
\eqref{eq:common-profile-transformations}. For instance, on
$\mathcal F_{w_0}$ with $w_0=1$,
\begin{align}
    c_C(w_0)\simeq1.109,
    \quad
    c_H(w_0)=1+0.60e^{-2}\simeq1.081,
\end{align}
so the three restricted logarithmic costs differ only by constant
vertical translations.

Figure~\ref{fig:dual-rotor-landscapes} shows visibly different
actuator-space cost landscapes with a common fiberwise-optimal locus.
The red curves represent selected fixed-task fibers, while the yellow
branches form the locus of the sets $\mathcal V^\star(w)$. Its
coincidence across the three panels illustrates
Theorem~\ref{thm:main} under the non-Euclidean SAC metric.

The determinant-proxy cost is even in each spin-rate coordinate,
\begin{align}
    L(\varepsilon_1v_1,\varepsilon_2v_2)
    =
    L(v_1,v_2),
    \quad
    \varepsilon_i\in\{-1,1\},
\end{align}
and therefore has four symmetry-related global minima. The four
squares in Fig.~\ref{fig:dual-rotor-landscapes}(a) mark these
minima, one of which is annotated at approximately
$(-1.83,-1.83)$. Since $f(-v)=-f(v)$ and neither $c_C$ nor $c_H$ is
even in $w$, the task-side modifications break this global symmetry.
The circle in panel~(b) marks a selected global minimizer of
$\widetilde L$ near $(1.76,1.73)$, while the diamond in panel~(c)
marks a selected global minimizer of $L_H$ near $(-1.77,-1.74)$.

Figure~\ref{fig:dual-rotor-summary} examines the same distinction
through fiber-restricted costs and fiber-optimum profiles. Panel~(a)
shows that the three costs restricted to $\mathcal F_{w_0}$ differ
only by constant vertical translations and have the same minimizers.
Panel~(b) displays the common fiberwise-optimal locus and locates the
global selections on it. Panel~(c) shows that the task-coordinate
stretching favors a task value near $w=5$, whereas the reference
metric favors one near $w=-5$. The disconnected branches of the
common locus also motivate the prescribed-task experiment below.

\begin{figure}[pos=t]
    \centering
    \includegraphics[width=\linewidth]{%
        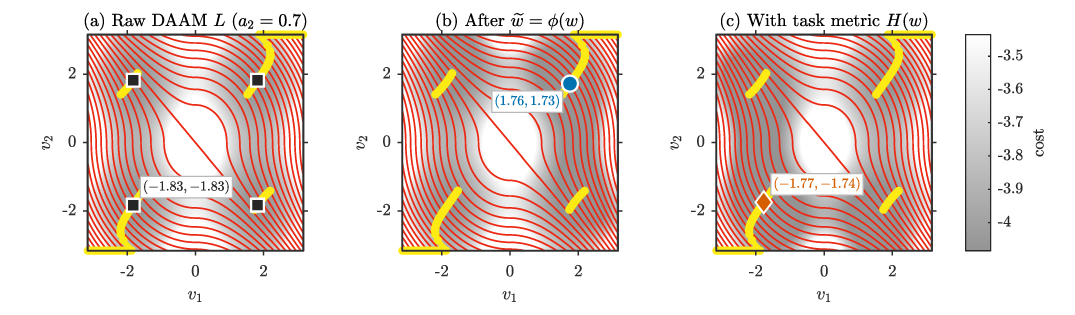}
    \caption{
    Actuator-space logarithmic-cost landscapes for the dual-rotor
    system. The grayscale background represents the logarithmic cost, with darker
gray indicating lower displayed values and white indicating higher
displayed values. The same cost scale is used across all three panels. (a) Determinant-proxy DAAM cost $L=-\log\rho$.
    (b) Cost $\widetilde L$ after the localized nonlinear
    task-coordinate stretching centered at $w=5$.
    (c) Cost $L_H$ obtained using the reference task metric, whose
    weighting is centered at $w=-5$. Red curves are fixed-task fibers,
    and the common yellow branches form the fiberwise-optimal locus.
    The four squares in panel~(a) mark the symmetry-related global
    minima of $L$. The circle and diamond mark selected global minima
    of $\widetilde L$ and $L_H$, respectively; the annotations report
    the corresponding spin-rate coordinates.
    }
    \label{fig:dual-rotor-landscapes}
\end{figure}

\begin{figure}[pos=t]
    \centering
    \includegraphics[width=0.97\linewidth]{%
        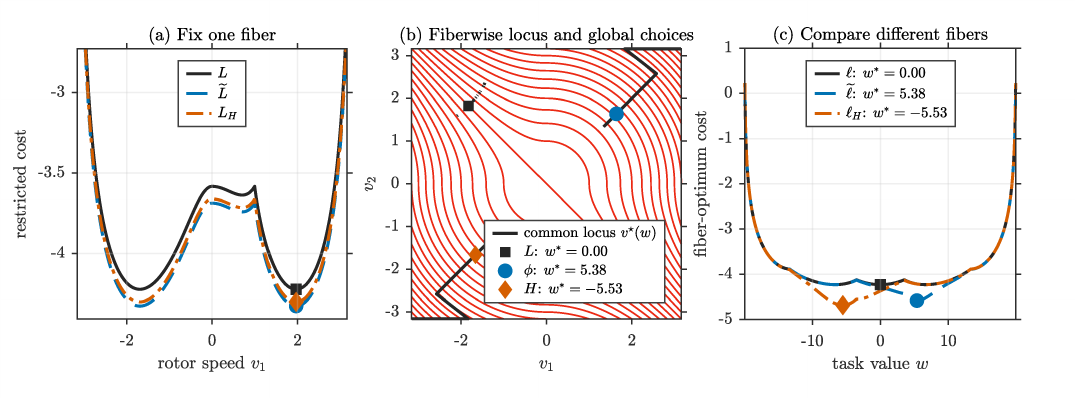}
    \caption{
    Fiberwise and global comparisons for the dual-rotor system.
    (a) Restrictions of $L$, $\widetilde L$, and $L_H$ to
    $\mathcal F_{w_0}$, parameterized by $v_1$; constant vertical
    translations preserve the minimizers. (b) Common fiberwise-optimal
    locus, with the square, circle, and diamond marking the global
    selections obtained from $L$, $\widetilde L$, and $L_H$,
    respectively. (c) Corresponding fiber-optimum profiles. The
    task-coordinate and task-metric constructions preserve the
    optimizer on every prescribed fiber but favor different fibers
    under global comparison.
    }
    \label{fig:dual-rotor-summary}
\end{figure}

\subsection{Allocation along a prescribed task trajectory}
\label{sec:prescribed-task-example}

We prescribe the smooth task trajectory
\begin{align}
    w(t)
    &=
    w_+
    +(w_--w_+)\left(3s^2-2s^3\right),
    \quad
    s=\frac{t}{T},
    \label{eq:dual-prescribed-task}
\end{align}
with $w_+=5.625$, $w_-=-w_+$, and $T=6$. Only the task trajectory is
prescribed. The initial and terminal rotor allocations are selected by
the optimizer subject to the instantaneous task constraint. An
internal-state trajectory $v(\cdot)$ satisfying
\begin{align}
    f(v(t))=w(t)
\end{align}
is a lift of the prescribed task trajectory.

\subsubsection{Pointwise allocation without temporal coupling}

We first consider the uncoupled admissible class $\mathcal A_w^0$
introduced in
Proposition~\ref{prop:prescribed-task-pointwise}. At every time $t$,
Theorem~\ref{thm:main} gives the common fiberwise-optimal set
\begin{align}
\begin{aligned}
    \mathcal V^\star(w(t))
    &:=
    \argmin_{u\in\mathcal F_{w(t)}}[-\rho(u)]
    =
    \argmin_{u\in\mathcal F_{w(t)}}[-\widetilde\rho(u)]
    =
    \argmin_{u\in\mathcal F_{w(t)}}[-\rho_H(u)]
    =
    \argmin_{u\in\mathcal F_{w(t)}}\ell_{\mathcal F}(u).
\end{aligned}
\label{eq:dual-common-pointwise-set}
\end{align}
The last equality follows from
Proposition~\ref{prop:fiber-normalized-range}: the minimum value of
$\ell_{\mathcal F}$ is zero and is attained precisely at the
fiberwise maximizers of manipulability.

Since $\mathcal A_w^0$ imposes no condition linking different times,
Proposition~\ref{prop:prescribed-task-pointwise} shows that the four
accumulated objectives have the same complete minimizer set. Their
minimizers are exactly the measurable selections satisfying
\begin{align}
    v(t)\in\mathcal V^\star(w(t))
    \quad
    \text{for almost every }t,
\end{align}
provided that such selections exist.

In this example, the common fiberwise-optimal locus has disconnected
branches, and no continuous selection can remain on that locus over
the complete prescribed task trajectory. A pointwise-optimal lift must
therefore contain instantaneous jumps between branches, as illustrated
in Fig.~\ref{fig:dual-rotor-pointwise-jumps}.

\begin{figure}[pos=t]
    \centering
    \includegraphics[width=0.8\linewidth,keepaspectratio]{%
        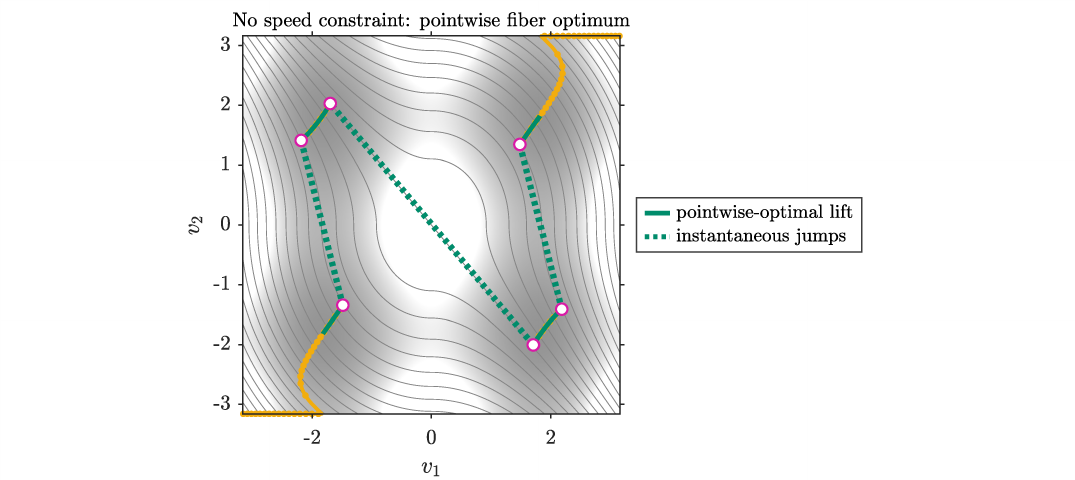}
   \caption{
Pointwise-optimal lift of the prescribed task trajectory without
temporal coupling. The grayscale background represents the raw
logarithmic cost, ranging from medium gray at lower values to white
at higher values. Thin gray curves are fixed-task fibers, while the
gold curves form the common fiberwise-optimal locus. Dark-teal
segments show a measurable selection from
$\mathcal V^\star(w(t))$; dotted connectors and magenta-outlined
markers identify the instantaneous jumps between disconnected
branches. Because different time instants are uncoupled, this
selection is simultaneously optimal for the instantaneous costs
$-\rho$, $-\widetilde\rho$, $-\rho_H$, and
$\ell_{\mathcal F}$.
}
    \label{fig:dual-rotor-pointwise-jumps}
\end{figure}

\subsubsection{Absolutely continuous allocation under an intrinsic speed bound}

The pointwise-optimal lift is not physically realizable because of its
instantaneous jumps. We therefore require absolute continuity and
impose an intrinsic speed bound through the SAC metric:
\begin{align}
    \mathcal A_{w,\bar s}
    :=
    \left\{
        v\in\mathrm{AC}([0,T],\mathcal E)
        \ \middle|\
        \begin{array}{l}
            f(v(t))=w(t)\quad\text{for every }t,\\
            \|\dot v(t)\|_{g_{v(t)}}\leq\bar s
            \quad\text{for almost every }t
        \end{array}
    \right\},
    \label{eq:dual-bounded-admissible-set}
\end{align}
where
\begin{align}
    \|\dot v\|_{g_v}
    &:=
    \sqrt{g_v(\dot v,\dot v)}
    =
    \sqrt{\dot v^\top G(v)\dot v},
    \quad
    \bar s=0.56.
    \label{eq:dual-intrinsic-speed}
\end{align}
The admissible class is intrinsic because both the task constraint and
the scalar $g_v(\dot v,\dot v)$ are independent of the coordinates
used on the internal-state manifold. Absolute continuity rules out
jumps, while the speed bound limits how rapidly a lift can move between
neighboring fibers.

The disconnected fiberwise-optimal branches make temporal coupling
operational. A feasible lift must leave one optimal branch before that
branch becomes unreachable under the speed bound, traverse
configurations with positive fiber-normalized loss, and approach the
next branch in anticipation of the forthcoming task values.

On the common admissible class
$\mathcal A_{w,\bar s}$, consider
\begin{align}
\begin{aligned}
    \mathcal J_\rho[v]
    &:=
    -\int_0^T\rho(v(t))\,\dd t,
    &
    \mathcal J_{\widetilde\rho}[v]
    &:=
    -\int_0^T\widetilde\rho(v(t))\,\dd t,
    \\
    \mathcal J_{\rho_H}[v]
    &:=
    -\int_0^T\rho_H(v(t))\,\dd t,
    &
    \mathcal J_{\mathcal F}[v]
    &:=
    \int_0^T\ell_{\mathcal F}(v(t))\,\dd t.
\end{aligned}
\label{eq:dual-prescribed-task-objectives}
\end{align}
Along the prescribed task trajectory, the three unnormalized
manipulability quantities satisfy
\begin{align}
    \widetilde\rho(v(t))
    &=
    c_C(w(t))\rho(v(t)),
    \quad
    \rho_H(v(t))
    =
    c_H(w(t))\rho(v(t)).
    \label{eq:dual-time-dependent-scalings}
\end{align}
At each fixed time, $c_C(w(t))$ and $c_H(w(t))$ are positive constants
on the instantaneous fiber, so they do not change its pointwise
optimizer. Along the complete trajectory, however, these factors vary
with $t$. Once the intrinsic speed bound couples allocations at
different times, the factors change the relative importance assigned
to different portions of a candidate lift. The three unnormalized
functionals can therefore favor different anticipatory transitions
between the common fiberwise-optimal branches.

By contrast, Theorem~
\ref{thm:intrinsic-cross-fiber-trajectory-optimization} applies
directly to $\mathcal J_{\mathcal F}$ on
$\mathcal A_{w,\bar s}$. The fiber-normalized functional is invariant
under coordinate changes on the internal-state and task manifolds and
independent of the reference task metric. Hence, every task-side
construction of $\ell_{\mathcal F}$ defines the same functional and
the same complete minimizer set on this common physical admissible
class.

\subsubsection{Numerical transcription}

For the numerical solution, $[0,T]$ is divided into $N=150$
intervals. At every time node $t_k$, the instantaneous fiber
$\mathcal F_{w(t_k)}$ is sampled at $1201$ values of $v_1$, and
$v_2$ is recovered analytically from
$f(v)=w(t_k)$. Samples on consecutive fibers are connected only when
the midpoint approximation of their SAC-metric displacement satisfies
\begin{align}
    \left(
        \frac{\Delta v_1}{s_1(\bar v_1)}
    \right)^2
    +
    \left(
        \frac{\Delta v_2}{s_2(\bar v_2)}
    \right)^2
    &\leq
    (\bar s\Delta t)^2,
    \label{eq:dual-discrete-speed-bound}
\end{align}
where
\begin{align}
    \bar v
    &:=
    \frac{v_k+v_{k+1}}{2},
    \quad
    \Delta v:=v_{k+1}-v_k,
    \quad
    \Delta t:=\frac{T}{N}.
\end{align}
A layered dynamic program minimizes a trapezoidal approximation of
each functional in
\eqref{eq:dual-prescribed-task-objectives} over the same transition
graph. Thus, all four numerical problems have identical discrete
feasible sets; differences among their solutions arise solely from
their objective values.

Figure~\ref{fig:dual-rotor-bounded-landscapes} shows the resulting
optimal lifts in the actuator-state plane. The three unnormalized
objectives select different transition corridors because their
task-dependent factors weight the visited fibers differently.
Panel~(d) superposes three independently computed lifts obtained from the equivalent task-side constructions of the intrinsic fiber-normalized loss;
all task-side constructions of $\ell_{\mathcal F}$ produce the same
complete discrete solution set.

Figure~\ref{fig:dual-rotor-speed-profiles} provides the corresponding
time-domain view. The pointwise solution jumps between disconnected
optimal branches, whereas the speed-bounded lifts leave the current
branch in advance and traverse suboptimal allocations to reach a
subsequent branch continuously.

\begin{figure}[pos=t]
    \centering
    \includegraphics[width=0.9\linewidth,keepaspectratio]{%
        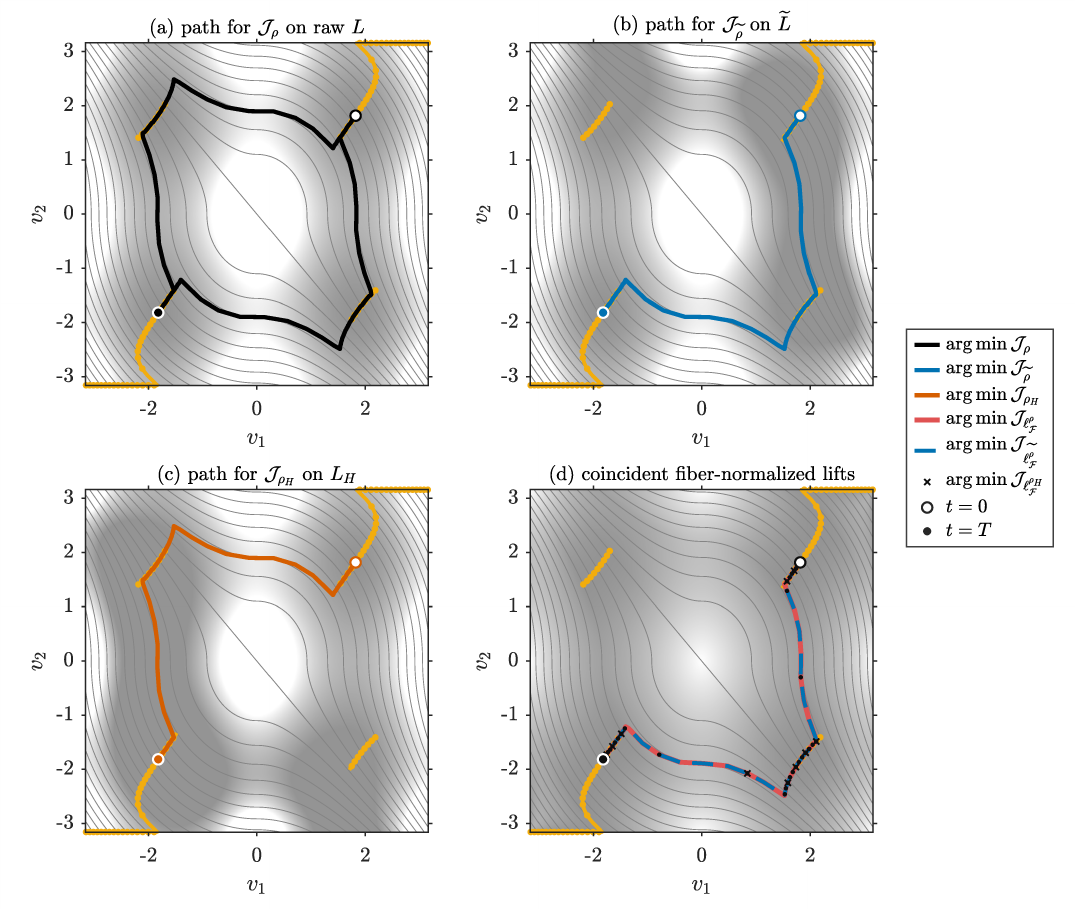}
    \caption{
    Optimal lifts of the same prescribed task trajectory under the
    intrinsic speed bound
    $\|\dot v\|_{g_v}\leq0.56$. In panels~(a)--(c), the grayscale background represents the
corresponding logarithmic running cost $L$, $\widetilde L$, or
$L_H$, while panel~(d) displays $\ell_{\mathcal F}$. Darker gray
indicates lower displayed cost and white indicates higher displayed
cost. Thin gray curves are fixed-task fibers, the gold curves form
the common pointwise fiberwise-optimal locus, and open and filled
circles identify $t=0$ and $t=T$, respectively.
Panels~(a)--(c) show representative minimizers of
$\mathcal J_\rho$, $\mathcal J_{\widetilde\rho}$, and
$\mathcal J_{\rho_H}$ over their corresponding running-cost
landscapes. Panel~(d) superposes the independently computed lifts
obtained by minimizing the functionals based on
$\ell_{\mathcal F}^{\rho}$,
$\ell_{\mathcal F}^{\widetilde\rho}$, and
$\ell_{\mathcal F}^{\rho_H}$.
Their pointwise coincidence illustrates the theoretically predicted
equality of the three fiber-normalized objectives and of their complete
minimizer sets.
    }
    \label{fig:dual-rotor-bounded-landscapes}
\end{figure}

\begin{figure}[pos=t]
    \centering
\includegraphics[width=0.95\linewidth,keepaspectratio]{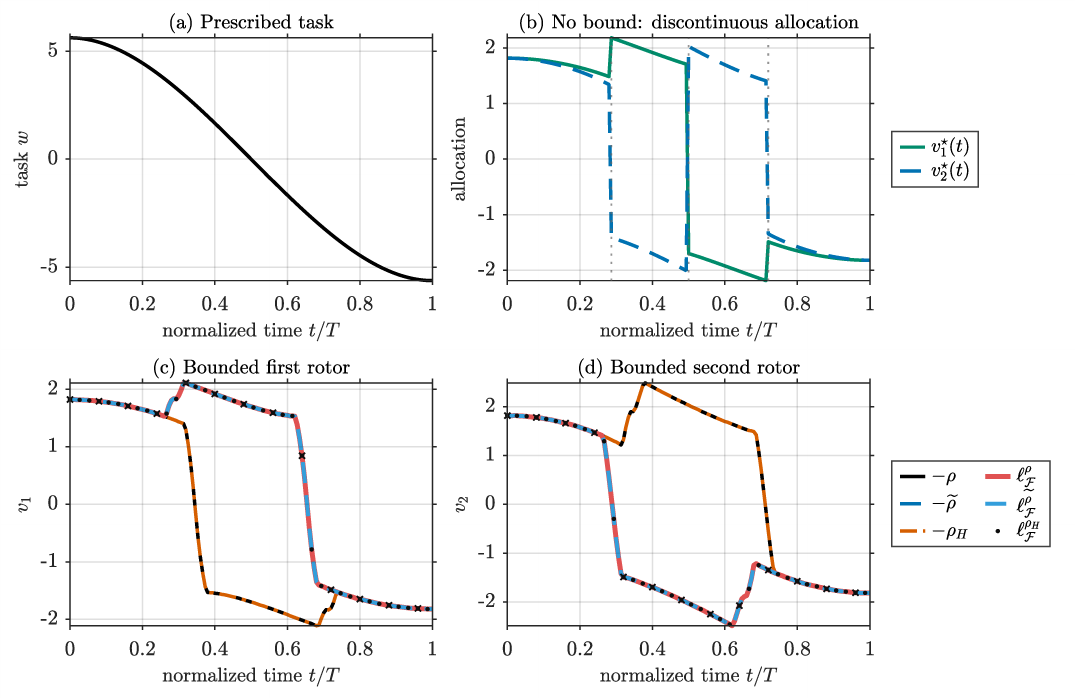}
    \caption{
Time-domain view of the prescribed-task allocation problem.
(a) Prescribed task trajectory $w(t)$.
(b) Pointwise-optimal allocation without temporal coupling; dotted
vertical lines mark instantaneous jumps between disconnected branches.
(c)--(d) Spin-rate profiles of the speed-bounded optimal lifts.
The intrinsic speed bound replaces the jumps with anticipatory
transitions through configurations having positive fiber-normalized
loss. The light-red solid, light-blue dashed, and black-marked
profiles are obtained by independently minimizing
$\ell_{\mathcal F}^{\rho}$,
$\ell_{\mathcal F}^{\widetilde\rho}$, and
$\ell_{\mathcal F}^{\rho_H}$, respectively.
Their pointwise coincidence confirms that the three task-side
constructions define the same objective and select the same optimal
lift. Any overlap between profiles produced by the unnormalized
objectives is instead a numerical symmetry of this example and is not
a general invariance property.
}
    \label{fig:dual-rotor-speed-profiles}
\end{figure}

\subsection{Free-task planning between fixed configurations}
\label{sec:free-task-example}

We finally consider a problem in which the endpoint internal states
are prescribed but the task trajectory is free. Without a motion
constraint, the endpoint-only continuous problem with a state-dependent
running cost is degenerate: an absolutely continuous trajectory can
compress the transitions toward and away from a favorable state into
arbitrarily short intervals. The numerical experiment is therefore
formulated explicitly as a finite graph problem.

Let $\V_h\subset\V$ be a finite grid, and let
$\mathcal E_h\subset\V_h\times\V_h$ contain the eight neighboring
moves and the zero move at every valid grid node. For a fixed number
of intervals $N$, define
\begin{align}
    \mathcal A_{A,B}^{N,h}
    &:=
    \left\{
        (v_0,\ldots,v_N)\in\V_h^{N+1}
        \ \middle|\
        \begin{array}{l}
            v_0=v_A,\quad v_N=v_B,\\
            (v_k,v_{k+1})\in\mathcal E_h,\\
            k=0,\ldots,N-1
        \end{array}
    \right\}.
    \label{eq:free-task-discrete-admissible-set}
\end{align}
For a node cost $C:\V_h\to\R$, use the trapezoidal functional
\begin{align}
    \mathcal J_C^{N,h}(v_0,\ldots,v_N)
    &:=
    \frac{\Delta t}{2}
    \sum_{k=0}^{N-1}
    \left[
        C(v_k)+C(v_{k+1})
    \right],
    \quad
    \Delta t=\frac{T}{N}.
    \label{eq:free-task-discrete-cost}
\end{align}
For either robotic system, the comparison uses
\begin{align}
    C
    \in
    \left\{
        -\rho,
        -\widetilde\rho,
        -\rho_H,
        \ell_{\mathcal F}
    \right\}.
\end{align}
Every objective uses the same grid, endpoints, horizon, edge set, and
treatment of zero moves. The edge rule in
\eqref{eq:free-task-discrete-admissible-set} acts as a common discrete
motion constraint; no convergence to a particular continuous
speed-bounded problem is claimed. When several minimum-cost paths
exist, the figures display one representative minimizer, while the
theoretical comparisons concern the complete solution sets.

Because the optimizer selects the task sequence $f(v_k)$, candidate
paths can visit different fibers at the same stage. The factors
$c_C(f(v_k))$ and $c_H(f(v_k))$ then assign different relative values
to graph nodes on different fibers and can produce different optimal
paths for the unnormalized objectives. This is a temporally coupled
cross-fiber problem and therefore lies outside the fixed-fiber
conclusion of Theorem~\ref{thm:main}.

By contrast, Theorems~\ref{thm:task-metric-independence}
and~\ref{thm:mu-coordinate-invariance} imply, at every node,
\begin{align}
    \ell_{\mathcal F}^{\rho}(v)
    &=
    \ell_{\mathcal F}^{\widetilde\rho}(v)
    =
    \ell_{\mathcal F}^{\rho_H}(v)
    =
    \ell_{\mathcal F}(v).
    \label{eq:free-task-normalized-node-equality}
\end{align}
Hence, the three normalized discrete functionals are identical on
the common finite set
\eqref{eq:free-task-discrete-admissible-set} and have exactly the same
complete solution set at every common grid resolution. This is the
finite-graph counterpart of
Theorem~\ref{thm:intrinsic-cross-fiber-trajectory-optimization}.

\subsubsection{Planar 2R manipulator}

For the 2R manipulator, $T=10$, and $\mathbb T^2$ is sampled on a
$181\times181$ periodic grid. The prescribed endpoint configurations
are
\begin{align}
    q_A
    &=
    (-2.40,0.30),
    \quad
    q_B
    =
    (0,\pi/4).
\end{align}
The transition graph contains the eight periodic neighboring moves
and the zero move. The number of intervals equals the minimum graph
distance between the selected endpoint nodes, so every admissible
discrete path connects them through neighboring grid states.

Figure~\ref{fig:2r-fixed-endpoint-paths} compares representative
minimizers of the four discrete problems. The unnormalized objectives
assign different relative values to states on different fibers and
can consequently select different routes. 

\begin{figure}[pos=t]
    \centering
    \includegraphics[width=0.9\linewidth,keepaspectratio]{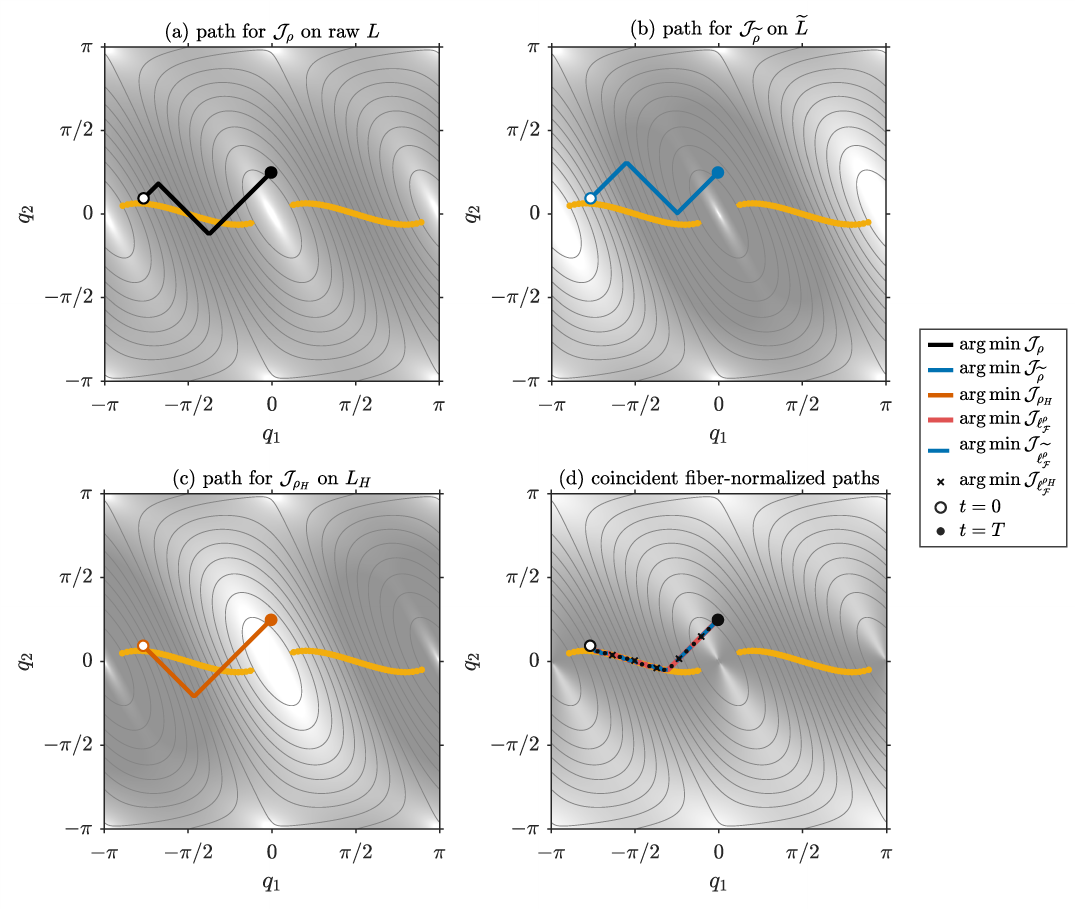}
    \caption{
    Free-task planning for the planar 2R manipulator between the same
    prescribed endpoint configurations. All panels use the same
    periodic grid, horizon, and transition graph. The grayscale backgrounds show $L$, $\widetilde L$, $L_H$, and
$\ell_{\mathcal F}$ in panels~(a)--(d), respectively, with darker
gray denoting lower displayed values. Thin gray curves are
fixed-task fibers and the gold curves form the common fiberwise-
optimal locus. Open and filled circles identify the prescribed
initial and final configurations. Panels~(a)--(c)
    show representative paths minimizing the discrete functionals
    based on $-\rho$, $-\widetilde\rho$, and $-\rho_H$,
    respectively. Panel~(d) superposes the independently computed paths obtained from
$\ell_{\mathcal F}^{\rho}$, 
$\ell_{\mathcal F}^{\widetilde\rho}$, and
$\ell_{\mathcal F}^{\rho_H}$. Their pointwise coincidence confirms
that fiber normalization removes both the task-coordinate and
task-metric dependence.
    }
    \label{fig:2r-fixed-endpoint-paths}
\end{figure}

\subsubsection{Dual rotor}

For the dual rotor, $T=10$, and the planning grid contains
$181\times171$ regularly sampled nodes before insertion of the exact
endpoint states. The prescribed endpoints are
\begin{align}
    v_A
    &=
    (-0.75\sqrt{10},0),
    \quad
    v_B
    =
    (0.75\sqrt{10},0).
\end{align}
The number of intervals equals the minimum graph distance between the
endpoints plus ten additional intervals. Every objective uses the
same valid task band, grid, endpoints, horizon, and edge set.

Figure~\ref{fig:dual-rotor-fixed-endpoint-paths} shows representative
minimizers. The determinant-proxy objective selects a lower corridor,
the task-coordinate modification selects an upper corridor, and the
task-metric construction favors a differently shaped lower passage.

Panel (d) superposes the three paths obtained by independently
minimizing the fiber-normalized objectives.
The normalized path need not coincide with an
unnormalized path: invariance concerns the different constructions of
$\ell_{\mathcal F}$.

\begin{figure}[pos=t]
    \centering
    \includegraphics[width=0.9\linewidth,keepaspectratio]{%
    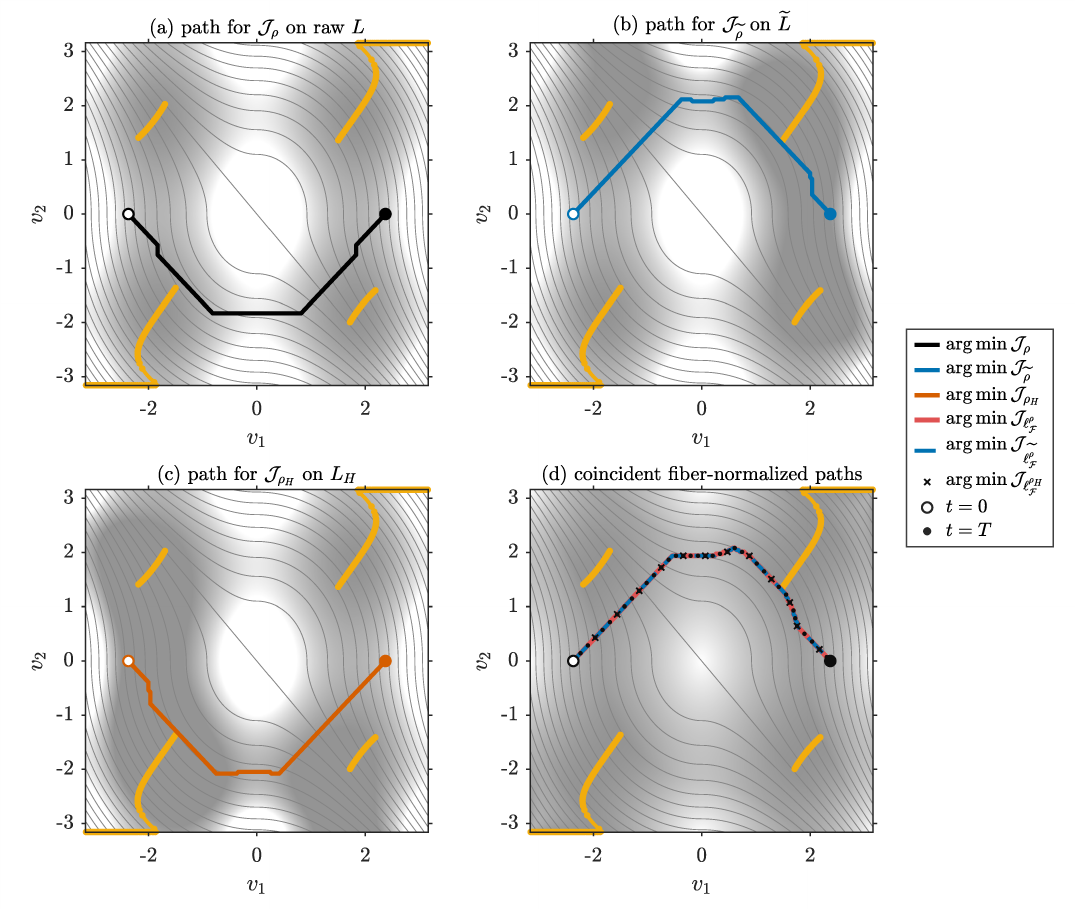}
    \caption{
    Free-task planning for the dual-rotor system between the same
    prescribed endpoint internal states. All panels use the same
    horizon, feasible grid, and transition graph.
    The grayscale backgrounds show $L$, $\widetilde L$, $L_H$, and
$\ell_{\mathcal F}$ in panels~(a)--(d), respectively, with darker
gray denoting lower displayed values. Thin gray curves are
fixed-task fibers and the gold curves form the common fiberwise-
optimal locus. Open and filled circles identify the prescribed
initial and final internal states. Panels~(a)--(c)
    show representative paths minimizing the discrete functionals
    based on $-\rho$, $-\widetilde\rho$, and $-\rho_H$,
    respectively. Panel~(d) superposes the independently computed paths obtained from
$\ell_{\mathcal F}^{\rho}$, 
$\ell_{\mathcal F}^{\widetilde\rho}$, and
$\ell_{\mathcal F}^{\rho_H}$. Their pointwise coincidence confirms
that fiber normalization removes both the task-coordinate and
task-metric dependence.
    }
    \label{fig:dual-rotor-fixed-endpoint-paths}
\end{figure}

Figure~\ref{fig:dual-rotor-fixed-endpoint-profile} evaluates the common
fiber-normalized loss along the three paths selected by the
unnormalized objectives and along the three independently computed
paths selected by the equivalent fiber-normalized constructions. The loss
assesses each visited internal state relative to the best allocation
available on its instantaneous fiber, rather than comparing absolute
manipulability values across different task values.

\begin{figure}[pos=t]
    \centering
    \includegraphics[width=0.85\linewidth,keepaspectratio]{%
        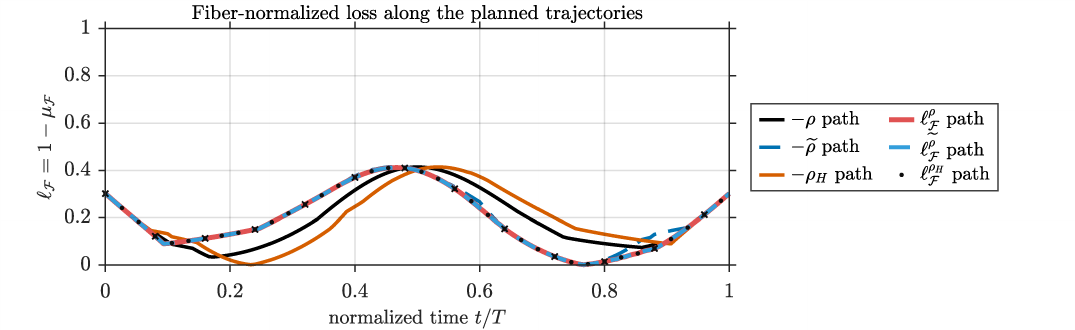}
    \caption{
Fiber-normalized loss $\ell_{\mathcal F}$ evaluated along the
representative paths in
Fig.~\ref{fig:dual-rotor-fixed-endpoint-paths}.
The black, blue, and orange curves correspond to paths obtained by
minimizing the unnormalized costs $-\rho$, $-\widetilde\rho$, and
$-\rho_H$, respectively.
The light-red solid, light-blue dashed, and black-marked curves
correspond to paths obtained by independently minimizing the
functionals based on
$\ell_{\mathcal F}^{\rho}$,
$\ell_{\mathcal F}^{\widetilde\rho}$, and
$\ell_{\mathcal F}^{\rho_H}$.
Their pointwise coincidence confirms that the three
fiber-normalized constructions define the same objective and select
the same optimal path.
}
    \label{fig:dual-rotor-fixed-endpoint-profile}
\end{figure}

\section{Discussion and Conclusions}
\label{sec:discussion}

This work separates three uses of determinant-based manipulability.
On a prescribed regular fiber, the determinant proxy $\rho$, its
representation in any task chart, and every metric-completed
manipulability $\rho_H$ differ only by positive constants. They
therefore induce the same complete ordering, constrained critical
points, gradient directions, and local optimality classifications.
Consequently, $\rho$ is an exact objective for fixed-task redundancy
optimization. The scaling changes gradient magnitudes, however, so
discrete algorithms using fixed gains or step sizes need not have
identical convergence rates.

Comparisons across fibers require additional care. The
metric-completed manipulability $\rho_H$ measures absolute
differential task capability relative to the internal metric $g$ and
the reference task metric $h$. By contrast, the fiber-normalized objective
$\mu_{\mathcal F}$ measures the fraction of the best capability
available on the current fiber that is attained by an internal state.
Thus, internal states $v_1$ and $v_2$ on different fibers may satisfy
\(
    \rho_H(v_1)
    <
    \rho_H(v_2),
    \quad
    \mu_{\mathcal F}(v_1)
    >
    \mu_{\mathcal F}(v_2).
\)
The first state then has lower absolute capability but better relative
redundancy utilization. Fiber normalization intentionally removes the
task-dependent capability scale, and all fiberwise-optimal states have
$\mu_{\mathcal F}=1$ even when their absolute capabilities differ.

The same distinction governs trajectory optimization. For a
prescribed task trajectory without temporal coupling, optimization
decomposes pointwise over the instantaneous fibers, and all
unnormalized and fiber-normalized objectives admit the same measurable
pointwise-optimal selections. Once continuity, intrinsic speed bounds,
endpoint conditions, or motion penalties couple different times, the
task-dependent factors multiplying the unnormalized objectives can
change the selected transitions. The same occurs when the task
trajectory is free. In these cross-fiber problems, the fiber-normalized loss
$\ell_{\mathcal F}$ provides an intrinsic state cost that can be
combined with physical motion constraints or intrinsic motion
penalties without introducing dependence on the task chart or the
reference task metric.

The planar-manipulator and dual-rotor examples illustrate these
statements. The visibly different unnormalized landscapes retain a
common optimizer locus on every prescribed fiber, but can select
different fibers under global comparison. Along a prescribed task
trajectory, disconnected pointwise-optimal branches produce jumps
when different times are uncoupled; an intrinsic speed bound replaces
these jumps with anticipatory transitions through suboptimal internal
states. In free-task graph planning, the unnormalized objectives can
select different routes, whereas every task-side construction of
$\ell_{\mathcal F}$ defines the same node cost and complete discrete
solution set.

These results suggest a direct choice of objective:
\begin{itemize}
    \item use $\rho$ for redundancy optimization on a prescribed
    regular fiber;
    \item use $\rho_H$ with physically meaningful $g$ and $h$ when
    absolute task capability is the intended quantity;
    \item use $\mu_{\mathcal F}$ or $\ell_{\mathcal F}$ when comparing
    relative redundancy utilization across task values or optimizing
    trajectories that cross fibers.
\end{itemize}

The analysis is restricted to regular regions of the task map and to
fibers having finite positive reference values. Near singularities,
the state-induced task co-metric loses positive definiteness, while
nonattained or nonunique fiberwise optima can compromise existence or
regularity of $\mu_{\mathcal F}$. Moreover, invariance does not imply
model independence: the task map $f$ determines the fibers, and the
internal metric $g$ determines how internal capability and motion are
measured. Computing $\mu_{\mathcal F}$ also requires solving or
approximating a fiberwise optimization problem. Efficient online
approximations, continuation across changing optimizer branches, and
regularization near singularities remain important directions for
future work.

\bibliographystyle{elsarticle-num}
\bibliography{bib}

\begin{thebibliography}{10}
\expandafter\ifx\csname url\endcsname\relax
  \def\url#1{\texttt{#1}}\fi
\expandafter\ifx\csname urlprefix\endcsname\relax\def\urlprefix{URL }\fi
\expandafter\ifx\csname href\endcsname\relax
  \def\href#1#2{#2} \def\path#1{#1}\fi

\bibitem{Whitney1969}
D.~E. Whitney, Resolved motion rate control of manipulators and human
  prostheses, IEEE Transactions on Man-Machine Systems 10~(2) (1969) 47--53.
\newblock \href {https://doi.org/10.1109/TMMS.1969.299896}
  {\path{doi:10.1109/TMMS.1969.299896}}.

\bibitem{Liegeois1977}
A.~Li{\'e}geois, Automatic supervisory control of the configuration and
  behavior of multibody mechanisms, IEEE Transactions on Systems, Man, and
  Cybernetics 7~(12) (1977) 868--871.
\newblock \href {https://doi.org/10.1109/TSMC.1977.4309644}
  {\path{doi:10.1109/TSMC.1977.4309644}}.

\bibitem{NakamuraHanafusaYoshikawa1987}
Y.~Nakamura, H.~Hanafusa, T.~Yoshikawa, Task-priority based redundancy control
  of robot manipulators, The International Journal of Robotics Research 6~(2)
  (1987) 3--15.
\newblock \href {https://doi.org/10.1177/027836498700600201}
  {\path{doi:10.1177/027836498700600201}}.

\bibitem{Siciliano1990}
B.~Siciliano, Kinematic control of redundant robot manipulators: A tutorial,
  Journal of Intelligent \& Robotic Systems 3~(3) (1990) 201--212.
\newblock \href {https://doi.org/10.1007/BF00126069}
  {\path{doi:10.1007/BF00126069}}.

\bibitem{Yoshikawa1985}
T.~Yoshikawa, Manipulability of robotic mechanisms, The International Journal
  of Robotics Research 4~(2) (1985) 3--9.
\newblock \href {https://doi.org/10.1177/027836498500400201}
  {\path{doi:10.1177/027836498500400201}}.

\bibitem{Yoshikawa1985Control}
T.~Yoshikawa, Manipulability and redundancy control of robotic mechanisms, in:
  Proceedings of the IEEE International Conference on Robotics and Automation,
  Vol.~2, 1985, pp. 1004--1009.
\newblock \href {https://doi.org/10.1109/ROBOT.1985.1087283}
  {\path{doi:10.1109/ROBOT.1985.1087283}}.

\bibitem{Yoshikawa1985Dynamic}
T.~Yoshikawa, Dynamic manipulability of robot manipulators, Transactions of the
  Society of Instrument and Control Engineers 21~(9) (1985) 970--975.
\newblock \href {https://doi.org/10.9746/sicetr1965.21.970}
  {\path{doi:10.9746/sicetr1965.21.970}}.

\bibitem{Yoshikawa1990}
T.~Yoshikawa, Translational and rotational manipulability of robotic
  manipulators, in: Proceedings of the American Control Conference, 1990, pp.
  228--233.
\newblock \href {https://doi.org/10.23919/ACC.1990.4790733}
  {\path{doi:10.23919/ACC.1990.4790733}}.

\bibitem{KleinBlaho1987}
C.~A. Klein, B.~E. Blaho, Dexterity measures for the design and control of
  kinematically redundant manipulators, The International Journal of Robotics
  Research 6~(2) (1987) 72--83.
\newblock \href {https://doi.org/10.1177/027836498700600206}
  {\path{doi:10.1177/027836498700600206}}.

\bibitem{Doty1995}
K.~L. Doty, C.~Melchiorri, E.~M. Schwartz, C.~Bonivento, Robot manipulability,
  IEEE Transactions on Robotics and Automation 11~(3) (1995) 462--468.

\bibitem{PatelSobh2015}
S.~Patel, T.~Sobh, Manipulator performance measures---a comprehensive
  literature survey, Journal of Intelligent \& Robotic Systems 77~(3--4) (2015)
  547--570.
\newblock \href {https://doi.org/10.1007/s10846-014-0024-y}
  {\path{doi:10.1007/s10846-014-0024-y}}.

\bibitem{Iwatsuki1994}
N.~Iwatsuki, I.~Hayashi, T.~Ohta, Optimum motion control of a redundant robot
  with the objective function of dexterity, JSME International Journal, Series
  C: Dynamics, Control, Robotics, Design and Manufacturing 37~(3) (1994)
  581--587.
\newblock \href {https://doi.org/10.1299/jsmec1993.37.581}
  {\path{doi:10.1299/jsmec1993.37.581}}.

\bibitem{Jin2017}
L.~Jin, S.~Li, H.~M. La, X.~Luo, Manipulability optimization of redundant
  manipulators using dynamic neural networks, IEEE Transactions on Industrial
  Electronics 64~(6) (2017) 4710--4720.
\newblock \href {https://doi.org/10.1109/TIE.2017.2674624}
  {\path{doi:10.1109/TIE.2017.2674624}}.

\bibitem{Su2019}
H.~Su, S.~Li, J.~Manivannan, L.~Bascetta, G.~Ferrigno, E.~De~Momi,
  Manipulability optimization control of a serial redundant robot for
  robot-assisted minimally invasive surgery, in: Proceedings of the IEEE
  International Conference on Robotics and Automation, 2019, pp. 1323--1328.
\newblock \href {https://doi.org/10.1109/ICRA.2019.8793676}
  {\path{doi:10.1109/ICRA.2019.8793676}}.

\bibitem{HuberWollherr2019}
G.~Huber, D.~Wollherr, Efficient closed-form task space manipulability for a
  7-{DOF} serial robot, Robotics 8~(4) (2019) 98.
\newblock \href {https://doi.org/10.3390/robotics8040098}
  {\path{doi:10.3390/robotics8040098}}.

\bibitem{HuberWollherr2021}
G.~Huber, D.~Wollherr, Globally optimal online redundancy resolution for serial
  7-{DOF} kinematics along {$SE(3)$} trajectories, in: Proceedings of the IEEE
  International Conference on Robotics and Automation, 2021, pp. 7570--7576.
\newblock \href {https://doi.org/10.1109/ICRA48506.2021.9560810}
  {\path{doi:10.1109/ICRA48506.2021.9560810}}.

\bibitem{KadenThomas2019}
S.~Kaden, U.~Thomas, Maximizing robot manipulability along paths in
  collision-free motion planning, in: Proceedings of the International
  Conference on Advanced Robotics, 2019, pp. 105--110.
\newblock \href {https://doi.org/10.1109/ICAR46387.2019.8981591}
  {\path{doi:10.1109/ICAR46387.2019.8981591}}.

\bibitem{Shen2023}
H.~Shen, W.-F. Xie, J.~Tang, T.~Zhou, Adaptive manipulability-based path
  planning strategy for industrial robot manipulators, IEEE/ASME Transactions
  on Mechatronics 28~(3) (2023) 1742--1753.
\newblock \href {https://doi.org/10.1109/TMECH.2022.3231467}
  {\path{doi:10.1109/TMECH.2022.3231467}}.

\bibitem{Maric2019}
F.~Mari{\'c}, O.~Limoyo, L.~Petrovi{\'c}, T.~Ablett, I.~Petrovi{\'c}, J.~Kelly,
  Fast manipulability maximization using continuous-time trajectory
  optimization, in: Proceedings of the IEEE/RSJ International Conference on
  Intelligent Robots and Systems, 2019, pp. 8258--8264.
\newblock \href {https://doi.org/10.1109/IROS40897.2019.8968441}
  {\path{doi:10.1109/IROS40897.2019.8968441}}.

\bibitem{Franchi2026Coactivation}
A.~Franchi, Muscle coactivation in the sky: Geometry and pareto optimality of
  energy vs. aerodynamic promptness and multirotors as variable stiffness
  actuators, in: Proceedings of the 2026 International Conference on Unmanned
  Aircraft Systems (ICUAS), IEEE, Corfu, Greece, 2026.
\newblock \href {https://doi.org/10.1109/ICUAS69441.2026.11598607}
  {\path{doi:10.1109/ICUAS69441.2026.11598607}}.

\bibitem{Franchi2026AeroPromptness}
A.~Franchi, Aero-promptness: Drag-aware aerodynamic manipulability for
  propeller-driven vehicles (2026).
\newblock \href {http://arxiv.org/abs/2603.07998} {\path{arXiv:2603.07998}},
  \href {https://doi.org/10.48550/arXiv.2603.07998}
  {\path{doi:10.48550/arXiv.2603.07998}}.

\bibitem{ParkBrockett1994}
F.~C. Park, R.~W. Brockett, Kinematic dexterity of robotic mechanisms, The
  International Journal of Robotics Research 13~(1) (1994) 1--15.
\newblock \href {https://doi.org/10.1177/027836499401300101}
  {\path{doi:10.1177/027836499401300101}}.

\bibitem{ParkKim1998}
F.~C. Park, J.~W. Kim, Manipulability of closed kinematic chains, Journal of
  Mechanical Design 120~(4) (1998) 542--548.
\newblock \href {https://doi.org/10.1115/1.2829312}
  {\path{doi:10.1115/1.2829312}}.

\bibitem{Lachner2020}
J.~Lachner, V.~Schettino, F.~Allmendinger, M.~D. Fiore, F.~Ficuciello,
  B.~Siciliano, S.~Stramigioli, The influence of coordinates in robotic
  manipulability analysis, Mechanism and Machine Theory 146 (2020) 103722.
\newblock \href {https://doi.org/10.1016/j.mechmachtheory.2019.103722}
  {\path{doi:10.1016/j.mechmachtheory.2019.103722}}.

\bibitem{Vahrenkamp2012}
N.~Vahrenkamp, T.~Asfour, G.~Metta, G.~Sandini, R.~Dillmann, Manipulability
  analysis, in: Proceedings of the IEEE-RAS International Conference on
  Humanoid Robots, 2012, pp. 568--573.
\newblock \href {https://doi.org/10.1109/HUMANOIDS.2012.6651576}
  {\path{doi:10.1109/HUMANOIDS.2012.6651576}}.

\bibitem{XuZhanCao2018}
Q.~Xu, Q.~Zhan, X.~Cao, The effect of redundant degrees of freedom on
  manipulator's kinematic characteristics, in: Proceedings of the International
  Conference on Robotics, Control and Automation Engineering, 2018, pp.
  127--131.
\newblock \href {https://doi.org/10.1145/3303714.3303718}
  {\path{doi:10.1145/3303714.3303718}}.

\bibitem{Chiu1988}
S.~L. Chiu, Task compatibility of manipulator postures, The International
  Journal of Robotics Research 7~(5) (1988) 13--21.
\newblock \href {https://doi.org/10.1177/027836498800700502}
  {\path{doi:10.1177/027836498800700502}}.

\bibitem{KazerounianWang1988}
K.~Kazerounian, Z.~Wang, Global versus local optimization in redundancy
  resolution of robotic manipulators, The International Journal of Robotics
  Research 7~(5) (1988) 3--12.
\newblock \href {https://doi.org/10.1177/027836498800700501}
  {\path{doi:10.1177/027836498800700501}}.

\bibitem{Zhang2020}
H.~Zhang, Q.~Sheng, Y.~Sun, X.~Sheng, Z.~Xiong, X.~Zhu, A novel coordinated
  motion planner based on capability map for autonomous mobile manipulator,
  Robotics and Autonomous Systems 129 (2020) 103554.
\newblock \href {https://doi.org/10.1016/j.robot.2020.103554}
  {\path{doi:10.1016/j.robot.2020.103554}}.

\end{thebibliography}

\end{document}